\documentclass{article} % For LaTeX2e
\usepackage{iclr2027_conference,times}

\usepackage{amsmath,amsfonts,bm}

\def\eqref#1{equation~\ref{#1}}
\def\1{\bm{1}}

\DeclareMathAlphabet{\mathsfit}{\encodingdefault}{\sfdefault}{m}{sl}
\SetMathAlphabet{\mathsfit}{bold}{\encodingdefault}{\sfdefault}{bx}{n}

\usepackage{hyperref}
\usepackage{url}
\usepackage{amsmath}
\usepackage{amssymb}
\usepackage{amsthm}
\usepackage{mathrsfs}
\usepackage{graphicx}
\usepackage{enumitem}
\usepackage{booktabs}
\usepackage{booktabs}
\usepackage{wrapfig}
\usepackage{makecell}
\usepackage{subcaption}
\usepackage[dvipsnames, table]{xcolor}
\usepackage[most]{tcolorbox}
\usepackage{enumitem}
\newcommand{\wa}{\mathbf{W}_1}
\newcommand{\wb}{\mathbf{W}_2}

\newtheorem{theorem}{Theorem}[section]
\newtheorem{proposition}[theorem]{Proposition}
\newtheorem{lemma}[theorem]{Lemma}
\newtheorem{corollary}[theorem]{Corollary}

\definecolor{darkgreen}{RGB}{0,100,0}
\definecolor{oursblue}{RGB}{220,240,250}
\newtheorem{definition}[theorem]{Definition}

\theoremstyle{remark}
\newtheorem{remark}[theorem]{Remark}

\title{$\lambda$-JEPA: Spectral Anti-Collapse \\ Regularization for Self-Supervised Learning}
\author{
\begin{tabular}{@{}l@{\qquad}l@{\qquad}l@{}}
\textbf{Berker Demirel}\textsuperscript{*} &
\textbf{Clémentine Dominé}\textsuperscript{*} &
\textbf{Valentino Maiorca} \\
\textbf{Marco Fumero} &
\textbf{Marco Mondelli} &
\textbf{Francesco Locatello}
\end{tabular}
\\[3pt]
Institute of Science and Technology Austria (ISTA)\\
Am Campus 1, 3400 Klosterneuburg, Austria
}

\iclrfinalcopy % Uncomment for camera-ready version, but NOT for submission.
\begin{document}
% SAC-JEPA: From Feature-Learning Dynamics to Spectral Anti-Collapse Regularization for JE-SSL

% SAC-JEPA : A Spectral Anti-collapse feature learnting insipire regulariser for SSL 

% SACReg: Spectral Anti-collapse regularizer for JE-SSL

% Feature learning to JE-SSL: Spectral anti collapse regularizer
 
\maketitle
\lhead{Preprint.}
\renewcommand{\thefootnote}{*}
\footnotetext{Equal contribution.}
\renewcommand{\thefootnote}{\arabic{footnote}}

\begin{abstract}
Joint-embedding self-supervised learning typically combines an invariance objective across augmented views with additional mechanisms to prevent representational collapse. These objectives are often applied after a projection head, while downstream tasks use the backbone representation before the projector. We find that this mismatch does not necessarily prevent dimensional collapse in the backbone, which can retain low effective rank and potentially limit downstream transfer. To address this, we introduce \texttt{SACReg}, a spectral anti-collapse regularizer motivated by an analysis of $\lambda$-balance, which captures the relative scale of weight matrices across layers. In a two-layer linear network, we show that \emph{(i)} $\lambda$-balance prevents collapse, and \emph{(ii)} our regularizer applied to the backbone induces $\lambda$-balance. In the nonlinear case, this regularizer leads to anti-collapse as well and, in realistic architectures on ImageNet100, it empirically increases the representations' ranks. We apply \texttt{SACReg} to JEPA and propose $\lambda$-JEPA, which improves over LeJEPA and VISReg on ImageNet-1k classification and in average linear-probe transfer performance across eight downstream image datasets. On video self-supervised learning, $\lambda$-JEPA improves over LeVJEPA and V-JEPA~2 on the Something-Something-v2 and Kinetics-400 benchmarks. Code is available at \href{https://github.com/berkerdemirel/lambda-jepa}{https://github.com/berkerdemirel/lambda-jepa}.

\end{abstract}

\section{Introduction}
\label{sec:intro}

% opener

Self-supervised learning aims to learn transferable representations from unlabeled data without requiring manual annotations~\citep{surveyjing2020self}. In computer vision, joint-embedding self-supervised learning (JE-SSL) has become one of the most promising approaches, encouraging representations of related views or regions of an image to agree while additional mechanisms prevent representational collapse~\citep{understandingjing2022, vicreg2022}. These mechanisms include using contrastive negatives~\citep{simclr2020, mocohe2020momentum}, asymmetric teacher-student training~\citep{dinocaron2021, ibotzhou2022, byolgrill2020} and explicit anti-collapse regularization~\citep{lejepa2025, vicreg2022, visreg2026}.

% observation + mismatch
JE-SSL methods typically optimize a projected representation during pretraining, while downstream tasks discard the non-linear projection head and use the backbone representation instead. These two representation spaces can differ substantially in both geometry and downstream performance~\citep{guillotinebordes2023}. %Prior work has shown that representation geometry and downstream performance can change substantially across layers~\citep{guillotinebordes2023}.
In contrastive self-supervised learning, this has also been studied through dimensional collapse, where learned representations span only a lower dimensional part of the available feature space~\citep{understandingjing2022}. We observe an analogous phenomenon in explicitly regularized JE-SSL methods. Although their anti-collapse objectives successfully maintain high rank projected representations, the corresponding backbone representations can remain low rank. 
This matters for task-agnostic representation learning, where different downstream tasks can benefit from different invariances~\citep{wellericsson2021, tian2020makes}. Dimensional collapse in the backbone can therefore limit transfer by reducing the set of feature directions available for downstream tasks. Consistent with this view, effective rank has been shown to predict downstream performance in JE-SSL~\citep{rankme2023}.

\looseness=-1The rank of learned representations has been studied extensively in the feature-learning literature, with prior work linking representation rank to distinct learning regimes in linear and nonlinear networks \citep{saxe_2014_exact,atanasov2022neural}.
%\CD{Should we cite Sueyon's papers here?} 
% Building on this perspective, we analyze a two-layer linear model and derive a regularizer that dynamically prevents collapse of the encoder representations while allowing the end-to-end mapping to adapt to the objective. \FL{This paragraph should contain the term lambda balance} 
Building on this perspective, we analyze a two-layer linear model with $\lambda$-balance initialization, which captures the relative scale of adjacent layers, and connects it to representation rank. The resulting spectral anti-collapse mechanism motivates a covariance regularizer for nonlinear networks, applied directly to the backbone features to restore their representation rank. In other words, our $\lambda$-balance theory prescribes both the form of the regularizer and where it should be applied. We then apply it accordingly to JE-SSL in a method we call $\lambda$-JEPA, using view-averaged backbone representations to prevent collapse and promote variation across images (see Figure \ref{fig:Schematic_SACReg}). In controlled experiments, we show that $\lambda$-JEPA improves class-relevant separation while preserving augmentation dependent variation. At larger scale, it outperforms both LeJEPA~\citep{lejepa2025} and VISReg~\citep{visreg2026} on ImageNet-1K~\citep{imagenetrussakovsky2015} trained from scratch, with performance approaching DINO~\citep{dinocaron2021} in both linear probing and transfer learning. The same approach applies off the shelf to video JEPA models, where $\lambda$-JEPA trained from scratch outperforms V-JEPA 2~\citep{vjepa2assran2025v} and LeVJEPA~\citep{levjepakuhn2026} on Something-Something-v2~\citep{ssv2goyal2017something} and Kinetics-400~\citep{kinetics400kay2017}.

% \BD{i think this is another paragraph if we can manage within the space}Beyond showing that our regularizer increases backbone rank and downstream performance, we study how these two effects are related. In particular, we characterize what variation is recovered in the additional representation directions, introducing \emph{augmentation thickness}, defined as the within-image variation across augmentations relative to variation between images. Through it, we find that regularized models retain more augmentation-dependent variation along downstream class directions. This shows that the performance gains are not simply a consequence of making different views more invariant, rather from recovering feature directions that remain useful for downstream prediction.

% Beyond showing that our regularizer restores backbone rank, we examine what variation is preserved in the representation. To measure this empirically, \MF{rephrased: "In addion we empirically measure what variations are preserved in the representation after applying our regularizer" (we should motivate here why we do this)} we introduce a metric we term \emph{augmentation thickness}, defined as the within-image variation across augmentations relative to variation between images. Through it, we find that regularized models retain more augmentation-dependent variation along downstream class directions. This shows that the performance gains are not simply a consequence of making different views more invariant, rather from recovering feature directions that remain useful for downstream prediction. 

We summarize our main contributions as:
\begin{itemize}[leftmargin=*]
    % \item We identify a backbone dimension collapse in explicitly regularized JE-SSL methods, showing that preventing collapse at the projected representation does not guarantee a high rank backbone. \BD{make it explicit why it is important}
    \item We show that standard JE-SSL methods designed to prevent collapse in the projected representation do not ensure a high-rank backbone, potentially limiting its representational capacity for downstream tasks.% (Section \ref{sec:related_work}).
    % \item We derive a spectral anti-collapse mechanism from a feature-learning analysis and extend it to nonlinear networks through a covariance regularizer applied directly to the backbone representation. \BD{needs to be more precise?} \MF{maybe write explicitly that the regularizer induce high rank representation (guaranteed, at least in the linear case) and better performance empirically (although this we can cross it out here if its the last bullet point)?}
   % \item From a feature learning analysis, we derive a spectral anti-collapse mechanism that guarantees non-collapse in the linear setting, and extend it to nonlinear networks with SACReg, which empirically increases backbone rank and improves downstream transfer.
    \item We derive SACReg, a spectral anti-collapse regularizer that guarantees encoder non-collapse in a two-layer linear network, motivating encoder-side spectral regularization. We extend this to nonlinear encoders by directly regularizing representation covariance and apply it to JE-SSL using view-averaged backbone representations to prevent collapse and promote variation across images. %(Section \ref{sec:method}). 
 %   \item We introduce augmentation thickness to study augmentation-dependent variation and use it to show the tradeoff between invariance and downstream transfer.
  % \CD{I would remove this, because it does no say much here}
    \item We propose $\lambda$-JEPA, a standalone SSL method that applies SACReg to both the backbone and projected representations. Across image JE-SSL benchmarks, $\lambda$-JEPA improves frozen transfer while remaining competitive on ImageNet-1k classification, and extends to video SSL with large gains on temporal recognition. %(Section \ref{sec:Exp} \ref{sec:Exp}).
    %\CD{Could add in the Thickness here }
\end{itemize}

\section{Problem and related works}
\label{sec:related_work}

\paragraph{Joint-embedding SSL} % or Self-supervised learning with projection heads ?
\label{sec:rel_jepa}
\begin{wrapfigure}[14]{r}{0.58\textwidth}
    \vspace{-0.8em}
    \centering
    \vspace{-0.5\baselineskip}
    \includegraphics[width=\linewidth]{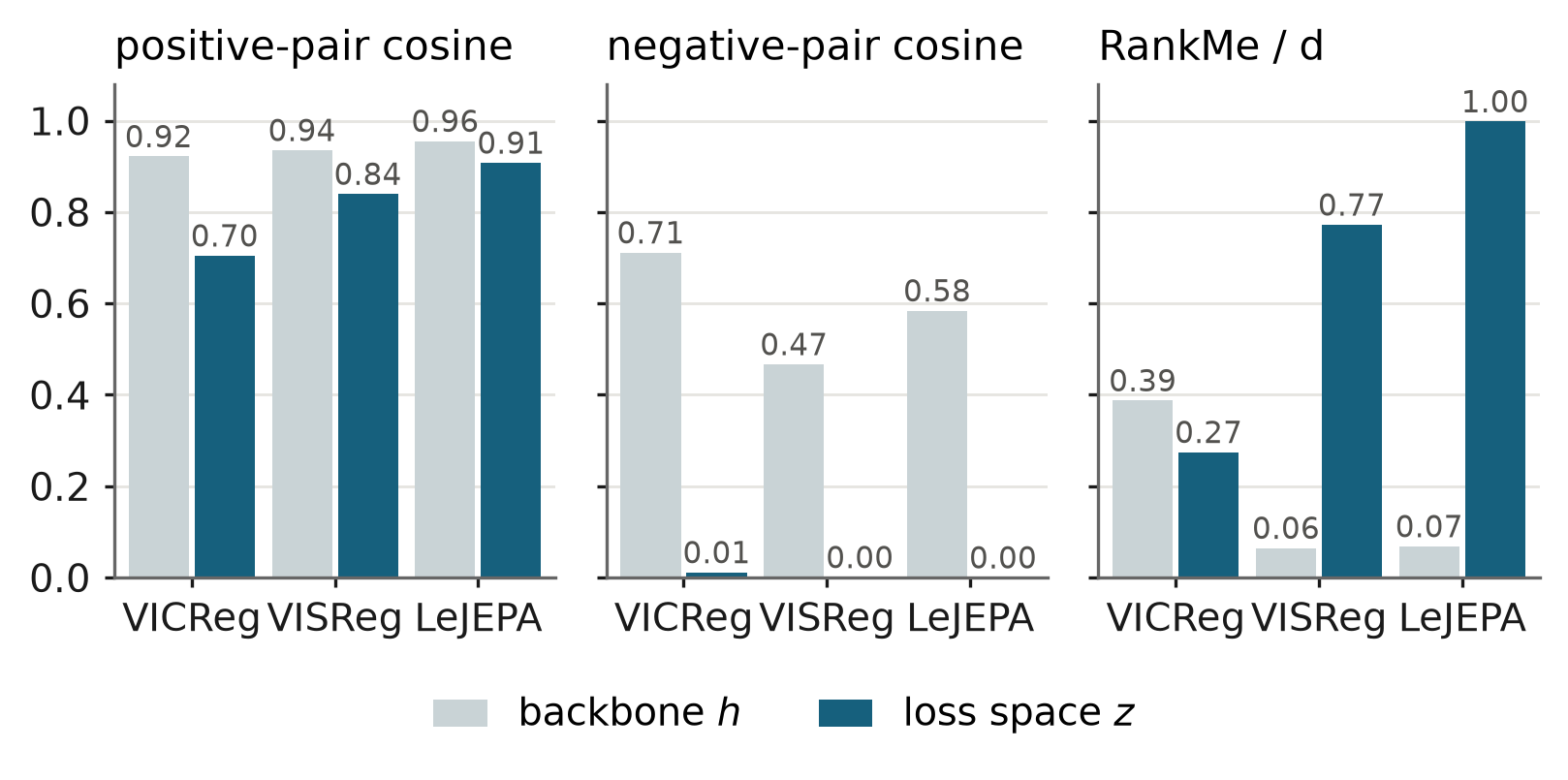}
    \vspace{-2.3em}
    \caption{Backbone and loss space geometry for explicitly regularized JE-SSL methods. The projection head changes both pairwise similarities and RankMe.}
    \label{fig:hz_geometry}
    \vspace{-0.5\baselineskip}
\end{wrapfigure}
Modern joint-embedding self-supervised learning (JE-SSL) methods differ primarily in how they avoid collapsed solutions. We focus on a family of explicitly regularized methods (VICReg~\citep{vicreg2022}, LeJEPA~\citep{lejepa2025}, VISReg~\citep{visreg2026}). In the projected space, VICReg controls feature variance and covariance, while LeJEPA regularizes the embedding distribution toward an isotropic Gaussian. VISReg similarly combines variance regularization with an additional distributional constraint. Following the use of projection heads in SimCLR~\citep{simclr2020}, these objectives act after a non-linear projection head, which is discarded for downstream tasks.

% \FL{the first sentence is not super clear}Subsequent work has further shown that transfer quality and representation geometry can differ substantially across the projector \citep{guillotinebordes2023,xue2024,ouyang2025}.

% We focus on a specific consequence of this separation: in these explicitly regularized
% methods, anti-collapse constraints are imposed on $\mathbf{z}$, yet with nonlinear
% projectors, preventing dimensional collapse in $\mathbf{z}$ does not in general prevent it in the backbone representation $h$.\footnote{For example, let $h=(X,0)^\top$ with $X$ uniform on $\{-1,0,1\}$, and let $p(h_1,h_2)=(h_1,|h_1|)^\top$. Then $\operatorname{Cov}(h)$ has rank one, while
% $\operatorname{Cov}(p(h))=\operatorname{diag}(2/3,2/9)$ has full rank.}.

We focus on a specific consequence of this separation: in these explicitly regularized
methods, anti-collapse constraints are imposed on $\mathbf{z}$, yet with nonlinear
projectors, preventing dimensional collapse in $\mathbf{z}$ does not in general prevent it in the backbone representation $\mathbf{h}$.\footnote{For example, let $\mathbf{h}=(X,0)^\top$ with $X$ uniform on $\{-1,0,1\}$, and let $p(\mathbf{h})=(h_1,|h_1|)^\top$. Then $\operatorname{Cov}(\mathbf{h})$ has rank one, while
$\operatorname{Cov}(p(\mathbf{h}))=\operatorname{diag}(2/3,2/9)$ has full rank.}. Figure~\ref{fig:hz_geometry} shows how two key aspects of explicitly regularized JE-SSL methods change across the projection head: (i) agreement between augmented views and (ii) prevention of representational collapse. Positive pairs remain highly aligned in both the backbone and projected spaces. In contrast, different images have very similar representations in the backbone, indicating substantial concentration. The same discrepancy appears in representation rank. RankMe~\citep{rankme2023} measures how evenly variance is distributed across feature directions and is typically higher in the projected space, whereas backbone remains low rank, especially for VISReg and LeJEPA. For task-agnostic representation learning, low rank backbones can be restrictive because different downstream tasks may rely on different feature directions~\citep{wellericsson2021,tian2020makes}. Consistently, ~\citet{rankme2023} have shown that it correlates with downstream performance in JE-SSL, motivating backbone-level anti-collapse regularization.

\vspace{-0.5em}
\paragraph{Feature learning and anti-collapse of the backbone representations}

The question of what determines the rank of learned representations has been studied extensively in the feature learning literature which has so far had little contact with SSL but offers directly relevant tools. 
In particular, deep linear networks provide a tractable setting for studying these effects. Despite their linear end-to-end mapping, their factorized parameterization induces nonlinear optimization dynamics ~\citep{baldi1989neural,fukumizu1998effect,saxe_2014_exact,du2019gradient,jacot2018neural,chizat2019lazy,braun2022exact} and captures phenomena observed in nonlinear networks~\citep{saxe2019mathematical,kunin2024get,nam2025position,anguita2026theory}. Initialization plays a central role in these dynamics: small initializations favor low-rank or sparse solutions, whereas appropriate large-scale limits yield kernel-like behavior~\citep{saxe_2014_exact,gunasekar2018characterizing,chizat2020implicit,woodworth2020kernel,li2020towards}.
In nonlinear networks, initialization scale likewise influences the learning regime~\citep{luo2021phase,atanasov2022neural}. 
%In particular, small initializations can promote feature learning, with neurons concentrating along a few directions in the condensed regime of two-layer ReLU networks~\citep{luo2021phase} and low-rank, task-aligned structure emerging in the neural tangent kernel~\citep{atanasov2022neural}.
Beyond overall initialization scale, the relative scales of adjacent layers, as studied through $\lambda$-balanced initializations, also influence learning dynamics and implicit bias~\citep{azulay2021implicit,kunin2024get,domine2025lazy,jarvis2025theory}. Their consequences for representation rank, however, remain less characterized, particularly in the nonlinear regime and in larger scale networks. Our work addresses this gap by investigating how layer balance shapes representation rank and using these insights to derive a regularizer that prevents collapse in the backbone representations of SSL models.
%Small-initialisation regimes have been linked to low-rank implicit bias in matrix factorisation \citep{li2020towards} ( of the netwrok fucntion)and sparsity-promoting implicit bias in diagonal linear networks \citep{woodworth2020kernel}.
%Related analyses show that small initialisations lead to approximately low-rank  , task-aligned layer weights in deep linear networks \citep{atanasov2022neural}.
\section{Method}
\label{sec:method}
Our goal is to prevent collapse in the backbone representations retained for downstream transfer. We first study a tractable two-layer linear network, where we show that negative layer balance yields a non-collapse guarantee
and motivates an encoder-side spectral regularizer.
We extend this principle to nonlinear encoders by directly regularizing representation covariance, then apply it to JE-SSL on view-averaged backbone representations to promote variation across images. Figure~\ref{fig:Schematic_SACReg} illustrates the resulting $\lambda$-JEPA objective and the placement of SACReg in the backbone and projected spaces.

%We then extend this principle to nonlinear encoders by regularizing the representation covariance directly. Finally, we apply the regularizer to joint-embedding self-supervised learning, targeting backbone representations averaged across augmented views to promote variation between images.

\vspace{-0.5em}
\subsection{Deriving a spectral anti-collapse mechanism in linear networks}

\textbf{Linear-network setting.}
Consider a supervised dataset
$
\mathcal{D}
=
\left\{
(\mathbf{x}_n,\mathbf{y}_n)
\right\}_{n=1}^{P}
$, $
\mathbf{x}_n\in\mathbb{R}^{N_i} 
$, $
\mathbf{y}_n\in\mathbb{R}^{N_o}.
$
We model the input--output mapping using a two-layer linear network, $\widehat{\mathbf{y}}_n
=
\wb\wa\mathbf{x}_n,$
where $\wa\in\mathbb{R}^{N_h\times N_i}$ and
$\wb\in\mathbb{R}^{N_o\times N_h}$ denote the encoder and decoder weight
matrices, respectively. The network is trained on the mean-squared-error %\BD{result is more general than MSE from appendix a.1.2 do you think we should mention that?}
loss
$
\mathcal{L}_{\mathrm{task}}(\wa,\wb)
=
\frac{1}{2}
\left\langle
\left\|
\wb\wa\mathbf{x}-\mathbf{y}
\right\|_2^2
\right\rangle,$
where $\langle\cdot\rangle$ denotes the empirical average over the training
set. Although the task loss depends only on the end-to-end mapping
$\mathbf{M}=\wb\wa$, its factorization across the two layers affects the parameter
and representation dynamics. We characterize this factorization through
the layer-imbalance matrix
$
\mathbf{\Delta}
=
\wb^\top\wb-\wa\wa^\top.
$
A network is said to be $\lambda$-balanced when
$
\mathbf{\Delta}
=
\lambda\mathbf{I}_{N_h}.$
Previous work has shown that $\lambda$-balanced initializations provide an analytically tractable framework for characterizing the transition between rich and lazy learning dynamics~\citep{domine2025lazy,nam2025position}. Under unregularized gradient flow, $\mathbf{\Delta}$ is conserved, so the layer balance remains fixed at its initial value. Prior studies have also linked these learning regimes to the rank of the representations \cite{saxe_2014_exact,atanasov2022neural}. Motivated by this connection, we investigate whether a suitable layer balance can prevent representation collapse. We show that an appropriately chosen negative balance provides such a guarantee, and use this result to derive a regularizer that dynamically promotes the same anti-collapse property during training.

\textbf{Balancedness anti-collapse property.}
A negative layer balance provides an explicit anti-collapse guarantee.
If $\mathbf{\Delta}=\lambda_{\mathrm{bal}}\mathbf{I}_{N_h}$ with
$\lambda_{\mathrm{bal}}<0$, then
$
\wa\wa^\top\succeq-\lambda_{\mathrm{bal}}\mathbf{I}_{N_h},
 \text{ and }
\sigma_{\min}(\wa)\geq\sqrt{-\lambda_{\mathrm{bal}}},
$
so the encoder has full row rank. This also yields a guarantee for
the hidden representations under the assumption that, for some $\kappa_x >0$, the input
covariance $\boldsymbol{\Sigma}_x
=
\frac{1}{P}\sum_{n=1}^{P}
(\mathbf{x}_n-\bar{\mathbf{x}})
(\mathbf{x}_n-\bar{\mathbf{x}})^\top \succeq\kappa_x \mathbf{I}_{N_i}$, with $ 
\bar{\mathbf{x}}=\frac{1}{P}\sum_{n=1}^{P}\mathbf{x}_n$.
In fact, the centered covariance of the hidden representations
$\mathbf{h}_n=\wa\mathbf{x}_n$ satisfies $\mathbf{C}_h
=
\wa\boldsymbol{\Sigma}_x\wa^\top
\succeq
-\kappa_x \lambda_{\mathrm{bal}}\mathbf{I}_{N_h}.$
Therefore, negative layer balance prevents representation collapse, see Lemma~\ref{lem:balanced-non-collapse} in Appendix \ref{app:noncollapse} for a formal statement and proof.

Although a negative balanced initialization prevents collapse, extending the same initialization strategy to deep nonlinear networks is non-trivial. 
We therefore seek a regularized training objective that dynamically selects a negative layer balance, independently of its initial value. Motivated by this objective, we derive a regularizer that follows from a constrained optimization problem. 

\textbf{Layer balance through regularization.} Among all factorizations of a fixed end-to-end mapping
$\mathbf{M}=\wb\wa$, we look for a constrained minimizer that satisfies
$
\wb^\top\wb-\wa\wa^\top
=
-\frac{\gamma}{\lambda_{\mathrm{reg}}}
\mathbf{I}_{N_h}.$

%The regulariser therefore selects a particular $\lambda$-balanced factorisation, rather than merely penalising the magnitude of the weights.

\begin{theorem}
\label{thm:encoder-variational-balance}
Fix an end-to-end mapping $
\mathbf{M}\in\mathbb{R}^{N_o\times N_i}$,  $\gamma>0$ and $\lambda_{\mathrm{reg}}>0$.
Let $(\wa^\star,\wb^\star)$ be a solution of
\begin{equation}
\begin{aligned}
\min_{\wa,\wb}\quad&
\frac{\lambda_{\mathrm{reg}}}{2}
\left(
    \|\wa\|_F^2+\|\wb\|_F^2
\right)
-
\frac{\gamma}{2}
\log\det\left(\wa\wa^\top\right)\quad \text{subject to}\quad
\wb\wa=\mathbf{M},
\end{aligned}
\end{equation}
with
\(\wa^\star(\wa^\star)^\top\succ0\). Then,  $(\wb^\star)^\top\wb^\star
-
\wa^\star(\wa^\star)^\top
=
-\frac{\gamma}{\lambda_{\mathrm{reg}}}
\mathbf{I}_{N_h}$
\end{theorem}
The proof of Theorem~\ref{thm:encoder-variational-balance} is provided in Appendix \ref{app:opptimisation}. This motivates the following definition.
\begin{tcolorbox}
\begin{definition}[Linear SACReg]
\label{def:linear-encoder-regulariser}
Assume that $\wa\wa^\top\succ0$, $\gamma >  0$ and $\lambda_{\mathrm{reg}} >  0$. We define
\[
\mathcal{R}_{\mathrm{LSAC}}(\wa,\wb)
=
\frac{\lambda_{\mathrm{reg}}}{2}
\left(
    \|\wa\|_F^2+\|\wb\|_F^2
\right)
-
\frac{\gamma}{2}
\log\det\left(\wa\wa^\top\right),
\]
%where $\lambda_{\mathrm{reg}}$ controls the strength of the symmetric
%$L_2$ decay and $\gamma>0$ controls the log-determinant term. 
and the
training objective  $\mathcal{L}_{\mathrm{LSAC}}(\wa,\wb)
=
\mathcal{L}_{\mathrm{task}}(\wa,\wb)
+
\mathcal{R}_{\mathrm{LSAC}}(\wa,\wb).
\label{eq:encoder-regularised-objective}$
\end{definition}
\end{tcolorbox}
\looseness=-1The regularizer combines symmetric $L_2$ weight decay with a log-determinant term acting specifically on the encoder Gram matrix. 
This placement is motivated by the non-collapse condition:
a negative layer balance bounds the encoder Gram matrix away from being singular and under the input-covariance assumption, guarantees non-collapse of the hidden representations. The log-determinant term
penalizes vanishing encoder singular values, while weight decay controls the overall weight scale.

We next note that the balance identified by the constrained optimization problem also emerges dynamically during training. If $N_h\leq N_i$ and $\wa(0)\wa(0)^\top\succ0$, the regularized gradient-flow solution for the mean-squared-error objective exists for all $t\geq0$, and the encoder retains full row rank throughout training. The task-loss terms cancel in the imbalance dynamics,
giving $\tau\dot{\mathbf{\Delta}}
=
-2\lambda_{\mathrm{reg}}\mathbf{\Delta}-2\gamma\mathbf{I}_{N_h}.$
Consequently,
$\mathbf{\Delta}(t)+\frac{\gamma}{\lambda_{\mathrm{reg}}}\mathbf{I}_{N_h}
=
e^{-2\lambda_{\mathrm{reg}}t/\tau}
\left(
\mathbf{\Delta}(0)+\frac{\gamma}{\lambda_{\mathrm{reg}}}\mathbf{I}_{N_h}
\right).$
Thus, independently of the dataset and initial imbalance, the network
converges exponentially with rate $2\lambda_{\mathrm{reg}}/\tau$ to  $\lambda_{\mathrm{bal}}^{\mathrm{enc}}$-balanced solution, with
$\lambda_{\mathrm{bal}}^{\mathrm{enc}}
=-\gamma/\lambda_{\mathrm{reg}}$ (see
Theorem ~\ref{thm:encoder-balance-dynamics} in Appendix \ref{app:balsel}).
The network preserves this balance throughout training (Corollary~\ref{cor:selected-balance-persistence}). Unlike $\lambda$-balanced initialization, which fixes the balance only
through the starting condition, the proposed objective continuously
attracts the network towards the selected balance. The ratio
$\gamma/\lambda_{\mathrm{reg}}$ determines the fixed point, while
$\lambda_{\mathrm{reg}}$ controls the rate of convergence towards it. 
%This construction connects regularization to the layer-balance
%mechanism studied in rich and lazy learning
%dynamics~\citep{domine2024lazy}; see
%Appendix~\ref{app:rich_lazy}.
The ratio $\gamma/\lambda_{\mathrm{reg}}$ determines the selected
balance and the corresponding lower bound on the encoder spectrum and yields a non-collapse guarantee.
This motivates applying the penalty directly to the
representation covariance in nonlinear encoders, as developed next.

% \begin{tcolorbox}[
%     colback=gray!14,
%     colframe=black,
%     boxrule=0.8pt,
%     arc=2mm,
%     left=3mm,
%     right=3mm,
%     top=1.5mm,
%     bottom=1.5mm
% ]
% \textbf{Linear anti-collapse mechanism.}
% Negative layer balance,
% \[
% \mathbf{\Delta}=\lambda_{\mathrm{bal}}\mathbf{I}_{N_h},
% \qquad
% \lambda_{\mathrm{bal}}<0,
% \]
% lower-bounds the encoder spectrum and prevents representation collapse.
% Linear SACReg dynamically selects
% \[
% \lambda_{\mathrm{bal}}^{\mathrm{enc}}
% =
% -\frac{\gamma}{\lambda_{\mathrm{reg}}},
% \]
% independently of the initial imbalance.
% \end{tcolorbox}

 %In particular, \citet{kunin2024get} show that both upstream and downstream initializations can induce rich dynamics when the overall initialization scale is small. Although balancedness is not conserved in nonlinear networks, the initial balance between layers still helps determine the learning regime. 
 %In the following section, we inpired by the linear setting we 
 %construct a practical regulariser for nonlinear networks, where exact balance dynamics are no longer available but the same spectral anti-collapse principle, the placement of the regulariser, and the learning regime ideas can be applied.

\subsection{Spectral regularization of nonlinear representations}
\label{sec:logdet-covariance-regulariser}
Although full matrix imbalance is not conserved in nonlinear networks, prior work has shown that it influences their learning dynamics~\citep{kunin2024get,anguita2026theory,jarvis2025theory}.
Motivated by this connection and our linear analysis, we apply
the spectral anti-collapse principle directly to the encoder's
representation covariance, preserving the encoder-side placement
without relying on exact balance dynamics. Let $\mathbf{h}=f_\theta(\mathbf{x})\in\mathbb{R}^{N_h}$ denote the representation
produced by a differentiable encoder $f_\theta$.
For a minibatch of size $B$, collect the representations row-wise
in $\mathbf{H}\in\mathbb{R}^{B\times N_h}$ and define $\boldsymbol{\mu}=\frac{1}{B}\mathbf{H}^\top\mathbf{1}_B,
$\text{ }$
\mathbf{P}_B=\mathbf{I}_B-\frac{1}{B}\mathbf{1}_B\mathbf{1}_B^\top,
$\text{ }$
\widetilde{\mathbf{H}}=\mathbf{P}_B\mathbf{H},
$\text{ }$
\mathbf{C}_h=\frac{1}{B} \widetilde{\mathbf{H}}^\top\widetilde{\mathbf{H}}.$
Here, $\boldsymbol{\mu}$ and $\mathbf{C}_h$ are the empirical representation mean
and centered covariance, respectively.
\begin{tcolorbox}
\begin{definition}[Nonlinear SACReg]
\label{def:nonlinear-encoder-regulariser}
Let $\lambda_{\mathrm{mean}}>0$,
$\lambda_{\mathrm{cov}}>0$, $\gamma>0$, and $\varepsilon\geq0$.
Assume that $\mathbf{C}_h+\varepsilon\mathbf{I}_{N_h}\succ0$.
We define 
%\BD{in case your theory do not repeat calling z as features, i would call it h to be consistent with the other parts, then SACReg becomes a function of H only as you already defined empirical mean and cov above?} 
\begin{equation}
\mathcal{R}_{\mathrm{SAC}}(\boldsymbol{\mu},\mathbf{C}_h)
=
\frac{\lambda_{\mathrm{mean}}}{2}\|\boldsymbol{\mu}\|_2^2
+
\frac{\lambda_{\mathrm{cov}}}{2}\operatorname{tr}(\mathbf{C}_h)
-
\frac{\gamma}{2}
\log\det\left(\mathbf{C}_h+\varepsilon\mathbf{I}_{N_h}\right),
\label{eq:nonlinear-covariance-regulariser_main}
\end{equation}
and the
corresponding training objective  $\mathcal{L}_{\mathrm{SAC}}
=
\mathcal{L}_{\mathrm{task}}
+
\mathcal{R}_{\mathrm{SAC}}(\boldsymbol{\mu},\mathbf{C}_h).$
\end{definition}
\end{tcolorbox}
The covariance terms play the same roles as in the linear setting: $\lambda_{\mathrm{cov}}$ controls the total representation variance, while $\gamma$ controls the strength of the log-determinant penalty on small covariance eigenvalues. The parameter $\varepsilon$ provides numerical stabilization. The nonlinear regularizer also includes a mean penalty $\|\boldsymbol{\mu}\|_2^2$, weighted by $\lambda_{\mathrm{mean}}$, which pulls the representation mean towards the origin. For a bias-free linear encoder, centered inputs produce centered representations, so no separate mean penalty is needed. %Because the trace and log-determinant terms depend only on centered covariance, they are unaffected by shifts in the representation mean. The additional $\|\boldsymbol{\mu}\|_2^2$ term thus controls this otherwise unconstrained shift.
For a linear encoder with whitened inputs, the covariance penalty reduces exactly to the encoder portion of the linear regularizer, with $\lambda_{\mathrm{cov}}=\lambda_{\mathrm{reg}}$. This connection is derived in Appendix~\ref{sec:nonlinear-covariance-regularisation}. 

\textbf{Anti-collapse property.} We next characterize the regularizer's inductive bias and anti-collapse properties. The mean and trace penalties control representation location and scale, while the negative log-determinant penalizes small covariance eigenvalues.%The trace and log-determinant terms play complementary roles. The trace penalty prevents unbounded growth of the representation variance, whereas
%the negative log-determinant penalises covariance spectra containing smalleigenvalues. Indeed, if
%$\lambda_1(\mathbf{C}_h),\ldots,\lambda_k(\mathbf{C}_h)$ are the eigenvalues of $\mathbf{C}_h$, then$\mathcal{R}_{\mathrm{SAC}}(\boldsymbol{\mu},\mathbf{C}_h)
%=
%\frac{\lambda_{\mathrm{mean}}}{2}\|\boldsymbol{\mu}\|^2
%+
%\frac{1}{2}
%\sum_{i=1}^{h}
%\left[
%    \lambda_{\mathrm{cov}}\lambda_i(\mathbf{C}_h)
%    -
  %  \gamma\log\left(\lambda_i(\mathbf{C}_h)+\varepsilon\right)
%\right].$
%Thus, the regulariser acts directly on the representation's location and on every direction of its covariance; see
%Proposition~\ref{prop:nonlinear-covariance-spectral-form} in the Appendix.
\begin{theorem}[Optimal mean and covariance]
\label{thm:nonlinear-optimal-covariance_main}
The unique minimizer of  \(\mathcal{R}_{\mathrm{SAC}}\) over $\boldsymbol{\mu}\in\mathbb R^{N_h},\ \mathbf{C}_h\succeq0$ such that $\mathbf{C}_h+\varepsilon\mathbf{I}_{N_h}\succ 0$ is $ 
\boldsymbol{\mu}^\star=\mathbf{0}$, $\mathbf{C}_h^\star
=
\left(
    \frac{\gamma}{\lambda_{\mathrm{cov}}}
    -
    \varepsilon
\right)_{+}\mathbf{I}_{N_h}$, where $(a)_{+}=\max\{a,0\}$. \end{theorem}
Thus, the regularizer favors zero-mean representations with isotropic covariance, which is full rank when $\gamma>\lambda_{\mathrm{cov}}\varepsilon$. For $\varepsilon=0$, the preferred variance is $\gamma/\lambda_{\mathrm{cov}}$. These are the preferred moments of the regularizer alone; the task loss, network parameterization, and minibatch rank constraints may prevent their attainment.

When $\varepsilon=0$, the negative log-determinant is a strict barrier against covariance collapse: $\lambda_{\min}(\mathbf{C}_h)\rightarrow0$ implies $\mathcal{R}_{\mathrm{SAC}}(\boldsymbol{\mu},\mathbf{C}_h)\rightarrow+\infty.$
At the other extreme, the trace term prevents unbounded growth: $\lambda_{\max}(\mathbf{C}_h)\rightarrow+\infty$ implies $\mathcal{R}_{\mathrm{SAC}}(\boldsymbol{\mu},\mathbf{C}_h)\rightarrow+\infty.$ The mean term supplies a third, independent barrier: $\|\boldsymbol{\mu}\|\rightarrow\infty$ implies $ \mathcal{R}_{\mathrm{SAC}}(\boldsymbol{\mu},\mathbf{C}_h)\rightarrow+\infty,$ so the representation cannot escape the regularizer by drifting off in mean instead of collapsing or exploding in covariance. 
Consequently, every finite sublevel set has bounded mean and covariance eigenvalues bounded above and away from zero; see Theorem~\ref{thm:strict-covariance-non-collapse} in Appendix \ref{app:logdet-covariance-regulariser}. We evaluate this property experimentally in nonlinear ReLU networks, following \cite{kunin2024get}. Consistent with our predictions, the regularizer increases representation rank across initialization scales (Fig. \ref{fig:rich-Relu} in Appendix \ref{app:Two-layer}).

Altogether, these results suggest that SACReg promotes higher-rank representations in nonlinear networks, supporting its role as an anti-collapse regularizer while also preserving feature learning.
%The nonlinear covariance regulariser preserves the spectral mechanism identified by the linear analysis,  the negative log-determinant opposes contraction of low-variance representation directions, while the trace term controls the overall scale. %These experiments therefore provide qualitative support rather than establish a formal extension of the theory. 
% {\color{red}Things we should really check here did we apply the right regular + maybe I should add the VAE results + more chats about kunin }

% \begin{tcolorbox}[
%     colback=gray!14,
%     colframe=black,
%     boxrule=0.8pt,
%     arc=2mm,
%     left=3mm,
%     right=3mm,
%     top=1.5mm,
%     bottom=1.5mm
% ]
% \textbf{Nonlinear spectral regularization.}
% We apply the same anti-collapse principle directly to the representation covariance:
% \[
% \mathcal{R}_{\mathrm{SAC}}(\boldsymbol{\mu},\mathbf{C}_h)
% =
% \frac{\lambda_{\mathrm{mean}}}{2}\|\boldsymbol{\mu}\|_2^2
% +
% \frac{\lambda_{\mathrm{cov}}}{2}\operatorname{tr}(\mathbf{C}_h)
% -
% \frac{\gamma}{2}
% \log\det(\mathbf{C}_h+\varepsilon \mathbf{I}_{N_h}).
% \]
% Its preferred moments are
% \[
% \boldsymbol{\mu}^\star=0,
% \qquad
% \mathbf{C}_h^\star
% =
% \left(
% \frac{\gamma}{\lambda_{\mathrm{cov}}}
% -
% \varepsilon
% \right)_+ \mathbf{I}_{N_h}.
% \]
% Thus, when $\gamma>\lambda_{\mathrm{cov}}\varepsilon$, the preferred covariance is full-rank and isotropic.
% \end{tcolorbox}
\subsection{Spectral anti-collapse for JE-SSL}
\label{sec:ssl}
As shown in Fig.~\ref{fig:hz_geometry}, standard JE-SSL objectives
can leave backbone representations dimensionally collapsed.
To address this, we apply the spectral anti-collapse regularizer (SACReg),
introduced in Section~\ref{sec:logdet-covariance-regulariser},
to joint-embedding self-supervised learning (JE-SSL), targeting
variation between images to prevent representation collapse. Standard JE-SSL objectives enforce invariance and prevent collapse in the projected space $\mathbf{z}$. Motivated by our analysis of $\lambda$-balance and the resulting anti-collapse guarantees, our standalone SSL method, \mbox{$\lambda$-JEPA}, applies SACReg to both backbone and projected representations. For a batch of $B$ images, each with $V$ augmented views, let
$\mathbf{h}_i^{(j)}=f_\theta(\mathbf{v}_i^{(j)}),$ and $
\mathbf{z}_i^{(j)}=p_\phi(\mathbf{h}_i^{(j)})$
denote the backbone and projected representations, respectively.
Define their view centers by $\bar{\mathbf{h}}_i=\frac{1}{V}\sum_{j=1}^{V}\mathbf{h}_i^{(j)},$ $
\bar{\mathbf{z}}_i=\frac{1}{V}\sum_{j=1}^{V}\mathbf{z}_i^{(j)},$
and write
$\bar{\mathbf{H}}=\{\bar{\mathbf{h}}_i\}_{i=1}^{B}$,
$\bar{\mathbf{Z}}=\{\bar{\mathbf{z}}_i\}_{i=1}^{B}$, and
$\mathbf{Z}=\{\mathbf{z}_i^{(j)}\}_{i=1,\ldots,B;\,j=1,\ldots,V}$. For a batch of view centers
$\mathbf{R}=\{\mathbf{r}_i\}_{i=1}^{B}\subset\mathbb R^d$,
define its empirical mean and centered covariance as
$\boldsymbol{\mu}_R=\frac{1}{B}\sum_{i=1}^{B}\mathbf{r}_i,$ and
$\boldsymbol{\Sigma}_R=\frac{1}{B}\sum_{i=1}^{B}
(\mathbf{r}_i-\boldsymbol{\mu}_R)(\mathbf{r}_i-\boldsymbol{\mu}_R)^\top.$
We average over $V$ views so that SACReg mainly encourages variation across images rather than augmentations.
\begin{tcolorbox}
\begin{definition}[SACReg for JE-SSL]
\label{def:sacreg-je-ssl}
For $\varepsilon\geq0$ such that
$\boldsymbol{\Sigma}_R+\varepsilon \mathbf{I}_d\succ0$, we define
\[
\mathrm{SACReg}(\mathbf{R})
=
\frac{1}{2}
\left[
\|\boldsymbol{\mu}_R\|_2^2
+\operatorname{tr}(\boldsymbol{\Sigma}_R)
-\log\det(\boldsymbol{\Sigma}_R+\varepsilon \mathbf{I}_d)
-d
\right].
\]
Given an invariance objective $\mathcal L_{\mathrm{SSL}}(\mathbf{Z})$
and weights $\beta_h,\beta_z>0$, the $\lambda$-JEPA objective is
\begin{equation}
\label{eq:lambda-jepa}
\mathcal L_{\lambda\text{-JEPA}}
=
\mathcal L_{\mathrm{SSL}}(\mathbf{Z})
+\beta_h\,\operatorname{SACReg}(\bar{\mathbf{H}})
+\beta_z\,\operatorname{SACReg}(\bar{\mathbf{Z}}).
\end{equation}
%SACReg is evaluated using the dimension of each representation space.
\end{definition}
\end{tcolorbox}
The mean term centers the representation, the trace controls its overall scale, and the negative log-determinant penalizes vanishing covariance directions. 
The projected-space term provides anti-collapse where the SSL objective is optimized, while the backbone term directly protects the representation retained for downstream transfer. When applying our regularizer to an existing SSL method that already contains its own projected-space anti-collapse mechanism, we leave its original objective unchanged and add only the backbone term,
$\beta_h\operatorname{SACReg}(\bar{\mathbf{H}})$.
In practice, when the feature dimension exceeds the number of image centers in a minibatch, we evaluate the regularizer over orthogonal lower-dimensional slices of the representation, details are given in Appendix~\ref{app:slicing}.

%Another possible choice is the log-determinant regularizer, whose
%Bregmann divergence is so called the LogDet divergence. There are many ap-
%plications of the LogDet divergence such as metric learning and Gaussian
%graphical models . However, the log-determinant regularizer is less popular in online prediction and it is unclear how to derive general and non-trivial regretbounds when using the FTRL with the log-determinant regularizer, as posedas an open problem in \cite{}. 

% \end{remark}

\section{Experiments}
\label{sec:Exp}
We evaluate whether SACReg improves the performance and transferability of representations. We first report the performance of the standalone $\lambda$-JEPA objective on image and video benchmarks, then use controlled ImageNet-100~\citep{inet100tian2020contrastive} experiments to isolate the effect of backbone regularization across diverse SSL objectives and analyze how it changes representation geometry.

\vspace{-0.5em}
\subsection{Setup}
\label{sec:exp-setup}

We consider two experimental settings. Our main evaluation covers large-scale image and video SSL. For images, we train ViT-S and ViT-B~\citep{vitdosovitskiy2021an} on ImageNet-1k~\citep{imagenetrussakovsky2015} for 100 and 400 epochs from scratch. For video, we follow LeVJEPA~\citep{levjepakuhn2026} and train ViT-S and ViT-B from scratch on a class-balanced 20\% subset of Kinetics-710~\citep{kinetics710li2023uniformerv2}, for up to 1{,}085 epochs. Our controlled experiments use ViT-S on ImageNet-100~\citep{inet100tian2020contrastive} to study the effect of SACReg across different SSL objectives.

Due to compute limitations, we do not retune hyperparameters for the 400-epoch ImageNet-1k or 1{,}085-epoch video runs; we use these primarily to study scaling with extended training. For evaluation, we follow standard frozen feature protocols: the Lightly~\citep{lightlySusmelj} linear probe benchmark on ImageNet-1k, the VISReg~\citep{visreg2026} pipeline for transfer and the LeVJEPA pipeline for video. Details are given in Appendix~\ref{app:exp-details}. Whenever public checkpoints trained under a common codebase are available (OpenKnowledge AI\footnote{\href{https://huggingface.co/collections/OK-AI/imagenet-1k-self-supervised-vit-baselines}{https://huggingface.co/collections/OK-AI/imagenet-1k-self-supervised-vit-baselines}}), we evaluate them under the same pipeline, otherwise we report published numbers. We set the loss weights using the gradient-based calibration procedure described in Appendix~\ref{app:loss-weights}.

\begin{table*}[t]
\centering
\small

\begin{minipage}[t]{0.44\textwidth}
\centering
\phantomsubcaption
\label{tab:main-results}
\textbf{(a) Main: 100 epochs}\\[1mm]
{
\renewcommand{\arraystretch}{1.28}
\begin{tabular}{@{}llrr@{}}
\toprule
Method & Backbone & Linear & Transfer \\
\midrule
LeJEPA & ViT-S/16 & 62.5 & 68.7 \\
DINO   & ViT-S/16 & 70.0 & 75.7 \\
iBOT   & ViT-S/16 & 70.9 & 75.9 \\
\rowcolor{oursblue}
$\lambda$-JEPA & ViT-S/16
& 69.7 & 77.6 \\
\midrule
LeJEPA & ViT-B/16 & 69.7 & 74.2 \\
DINO   & ViT-B/16 & 73.9 & 78.3 \\
iBOT   & ViT-B/16 & 77.0 & 80.7 \\
VISReg & ViT-B/16 & 70.3 & 76.4 \\
\rowcolor{oursblue}
$\lambda$-JEPA & ViT-B/16
& 74.2 & 80.9 \\
\bottomrule
\end{tabular}
}
\end{minipage}
\hfill
\begin{minipage}[t]{0.53\textwidth}
\centering
\phantomsubcaption
\label{tab:long-training-results}
\textbf{(b) Longer training}\\[1mm]
\begin{tabular}{@{}llrrr@{}}
\toprule
Method & Backbone & Ep. & Linear & Transfer \\
\midrule
LeJEPA & ViT-S/16 & 300 & 66.4 & 71.9 \\
DINO   & ViT-S/16 & 300 & 73.8 & 79.4 \\
iBOT   & ViT-S/16 & 300 & 75.2 & 79.5 \\
\rowcolor{oursblue}
$\lambda$-JEPA & ViT-S/16 & 400
& 72.3 & 79.3 \\
\midrule
LeJEPA & ViT-B/16 & 300 & 72.4 & 75.2 \\
DINO   & ViT-B/16 & 300 & 75.2 & 79.3 \\
iBOT   & ViT-B/16 & 300 & 78.6 & 82.3 \\
MoCo v3$^\ddag$ & ViT-B/16 & 300 & 75.9 & 80.5 \\
DINO$^\ddag$    & ViT-B/16 & 400 & 77.2 & 83.1 \\
iBOT$^\ddag$    & ViT-B/16 & 400 & 78.5 & 83.2 \\
VISReg$^\ddag$  & ViT-B/16 & 400 & 74.6 & 79.1 \\
\rowcolor{oursblue}
$\lambda$-JEPA & ViT-B/16 & 400
& 76.0 & 82.0 \\
\bottomrule
\end{tabular}
\end{minipage}
\vspace{-0.5em}
\caption{
ImageNet-1k linear probing and mean linear-probe transfer over eight datasets.
\textbf{(a)} Main 100-epoch comparison.
\textbf{(b)} Longer-training results, including public SSL checkpoints for context.
$^\ddag$ Transfer results are reported by VISReg; all others are evaluated
by us under the matched VISReg protocol.
Detailed results are in Appendix~\ref{app:detailed-results}.
}
\vspace{-1.5em}
\label{tab:image-results}
\end{table*}

% % =========================
% % Main results: 100 epochs
% % =========================
% \begin{table}[t]
% \centering
% \small
% \begin{tabular}{llrr}
% \rowcolor{gray!20}\toprule
% Method & Backbone & Linear & Transfer Avg. \\
% \midrule
% LeJEPA & ViT-S/16 & 62.5 & 68.7 \\
% DINO   & ViT-S/16 & 70.0 & 75.7 \\
% iBOT   & ViT-S/16 & 70.9 & 75.9 \\
% \textbf{$\bm{\lambda}$-JEPA} & \textbf{ViT-S/16}
% & \textbf{69.7} & \textbf{77.6} \\
% \midrule
% LeJEPA & ViT-B/16 & 69.7 & 74.2 \\
% DINO   & ViT-B/16 & 73.9 & 78.3 \\
% iBOT   & ViT-B/16 & 77.0 & 80.7 \\
% VISReg & ViT-B/16 & 70.3 & 76.4 \\
% \textbf{$\bm{\lambda}$-JEPA} & \textbf{ViT-B/16}
% & \textbf{74.2} & \textbf{80.9} \\
% \bottomrule
% \end{tabular}
% \caption{
% ImageNet-1k 100-epoch results. ImageNet-1k linear probing accuracy,
% together with average linear-probe transfer accuracy over DTD, Aircraft,
% Cars, CIFAR10, CIFAR100, Flowers, Food and Pets. See Appendix~\ref{app:detailed-results} for
% detailed transfer results.
% }
% \label{tab:main-results}
% \end{table}
\vspace{-0.5em}
\subsection{Performance on image self-supervised learning}
Table~\ref{tab:main-results} reports our main 100-epoch results on ImageNet-1k linear probing and average transfer across eight classification datasets. We compare most directly with LeJEPA~\citep{lejepa2025} and VISReg~\citep{visreg2026}, the related explicitly regularized JE-SSL methods. We include DINO~\citep{dinocaron2021} and iBOT~\citep{ibotzhou2022} as established SSL baselines for broader comparison which, until now, have significantly outperformed JEPA-based models. At 100 epochs, where training length is matched, $\lambda$-JEPA substantially improves over prior explicitly regularized methods on both ImageNet-1k linear probing and transfer. The gains are particularly strong on transfer, improving over LeJEPA by 8.9 points with ViT-S and over VISReg by 4.5 points with ViT-B. $\lambda$-JEPA is also competitive with DINO and iBOT: its ImageNet-1k performance approaches these methods, while both model sizes achieve the highest average transfer among the 100-epoch results.

%At 100 epochs, where training length is matched, $\lambda$-JEPA substantially improves over the prior explicitly regularized methods. On ImageNet-1k linear probing, our ViT-S improves over LeJEPA by 7.2 points, while our ViT-B improves over LeJEPA and VISReg by 4.5 and 3.9 points, respectively. The gains are larger on the transfer benchmark. $\lambda$-JEPA improves 8.9 points over LeJEPA for ViT-S, and 6.7 and 4.5 points over LeJEPA and VISReg for ViT-B. Performance is also competitive with DINO and iBOT: ViT-S is within 0.3 points of DINO on ImageNet-1k linear probing, ViT-B slightly exceeds DINO, and both models achieve the highest transfer average among the 100-epoch results.

Table~\ref{tab:long-training-results} reports results with extended pretraining: we use the same recipe and hyperparameters as the 100-epoch experiments, without any additional tuning for the longer runs. Extending $\lambda$-JEPA to 400 epochs improves both ImageNet-1k linear probing and transfer performance. Since the available checkpoints use different training lengths, we include them as longer training comparisons rather than strictly matched baselines. With ViT-B, $\lambda$-JEPA reaches 82.0 average transfer, outperforming the 400-epoch VISReg result by 2.9 points and coming within 1.2 points of DINO and iBOT. With ViT-S, it reaches 79.3 average transfer, essentially matching the 300-epoch DINO and iBOT results.

\textbf{Takeaway.} $\lambda$-JEPA substantially improves over prior explicitly regularized JE-SSL methods on both ImageNet-1k classification and transfer. To the best of our knowledge, it is the first JEPA-based method to approach DINO and iBOT on both evaluations, while achieving the highest average transfer performance at 100 epochs.

% % =========================
% % Longer training
% % =========================
% \begin{table}[t]
% \centering
% \small
% \begin{tabular}{llrrr}
% \rowcolor{gray!20}\toprule
% Method & Backbone & Ep. & Linear & Transfer \\
% \midrule
% LeJEPA & ViT-S/16 & 300 & 66.4 & 71.9 \\
% DINO   & ViT-S/16 & 300 & 73.8 & 79.4 \\
% iBOT   & ViT-S/16 & 300 & 75.2 & 79.5 \\
% \textbf{$\bm{\lambda}$-JEPA} & \textbf{ViT-S/16} & \textbf{400}
% & \textbf{72.3} & \textbf{79.3} \\
% \midrule
% LeJEPA & ViT-B/16 & 300 & 72.4 & 75.2 \\
% DINO   & ViT-B/16 & 300 & 75.2 & 79.3 \\
% iBOT   & ViT-B/16 & 300 & 78.6 & 82.3 \\
% MoCo v3$^\ddag$ & ViT-B/16 & 300 & 75.9 & 80.5 \\
% DINO$^\ddag$    & ViT-B/16 & 400 & 77.2 & 83.1 \\
% iBOT$^\ddag$    & ViT-B/16 & 400 & 78.5 & 83.2 \\
% VISReg$^\ddag$  & ViT-B/16 & 400 & 74.6 & 79.1 \\
% \textbf{$\bm{\lambda}$-JEPA} & \textbf{ViT-B/16} & \textbf{400}
% & \textbf{76.0} & \textbf{82.0} \\
% \bottomrule
% \end{tabular}
% \caption{
% Longer-training results. ImageNet-1k linear probing accuracy, together
% with average linear-probe transfer accuracy. Longer-training public SSL
% checkpoints are included as context. $^\ddag$ Transfer numbers are
% reported by VISReg; other transfer results are evaluated by us under the
% matched VISReg protocol. See Appendix~\ref{app:detailed-results} for
% detailed transfer results.
% }
% \label{tab:long-training-results}
% \end{table}
% \CD{Can we put the table next to each other to save space?}

\vspace{-0.5em}
\subsection{Performance on video self-supervised learning}

\label{sec:exp-video}

\begin{wraptable}[21]{r}{0.38\columnwidth}
\vspace{-1.\baselineskip}
\centering
\footnotesize
\setlength{\tabcolsep}{3.5pt}

\begin{tabular}{@{}lrrr@{}}
\toprule
Method & IN-1k & SSv2 & K400 \\
\midrule

\multicolumn{4}{@{}l}{\textit{ViT-S/16, 240 epochs}} \\
\cmidrule(lr){1-4}
LeVJEPA & 39.4 & --- & --- \\
V-JEPA 2 & 38.7 & --- & --- \\
\rowcolor{oursblue}
$\lambda$-JEPA & 46.8 & 37.3 & 40.0 \\

\addlinespace[1pt]
\midrule
\multicolumn{4}{@{}l}{\textit{ViT-B/16, 240 epochs}} \\
\cmidrule(lr){1-4}
VideoMAEv2 & 47.1 & --- & --- \\
V-JEPA 2 & 51.6 & --- & --- \\
LeVJEPA & 50.7 & 30.4 & --- \\
\rowcolor{oursblue}
$\lambda$-JEPA & 52.9 & 43.9 & 44.2 \\

\addlinespace[1pt]
\midrule
\multicolumn{4}{@{}l}{\textit{ViT-B/16, longer training}} \\
\cmidrule(lr){1-4}
VideoMAEv2 & 53.4 & 43.6 & 37.4 \\
V-JEPA 2 & 51.6 & 42.5 & 40.7 \\
LeVJEPA & 61.0 & 40.4 & 44.6 \\
\rowcolor{oursblue}
$\lambda$-JEPA & 57.1 & 48.3 & 45.7 \\
\bottomrule
\end{tabular}

\caption{Video SSL on IN-1k, SSv2, and K400. Longer training LeVJEPA and $\lambda$-JEPA use
1085 epochs.
}
\label{tab:video-main}
\vspace{-0.5\baselineskip}
\end{wraptable}

We apply the same objective to video by replacing the loss inside the LeVJEPA~\citep{levjepakuhn2026} training pipeline while keeping the model, optimizer and data pipeline unchanged. We train these models from scratch and freeze the encoders after pretraining. The frozen encoders are evaluated on three benchmarks with different characteristics: ImageNet-1k for recognition from a single image, Something-Something-v2 (SSv2)~\citep{ssv2goyal2017something} for temporal reasoning and Kinetics-400~\citep{kinetics400kay2017} for action recognition.

Table~\ref{tab:video-main} compares $\lambda$-JEPA with LeVJEPA, V-JEPA 2~\citep{vjepa2assran2025v} and VideoMAEv2~\citep{vmaewang2023videomae}. At 240 epochs, $\lambda$-JEPA improves over LeVJEPA and V-JEPA 2 on ImageNet-1k for both ViT-S and ViT-B. The largest gains are on SSv2: our ViT-B improves over LeVJEPA by 13.5 points at 240 epochs and by 7.9 points at 1{,}085 epochs. For ImageNet-1k and SSv2, LeVJEPA evaluates frozen features with an attentive probe followed by a linear classifier, whereas on K400 it reports a linear probe on mean-pooled features. In our setting, SACReg is applied to the CLS representation, so we report a CLS probe in Table~\ref{tab:video-main}, evaluating the representation directly regularized during pretraining. It reaches 45.7 at 1{,}085 epochs compared with 44.6 for LeVJEPA's probe over the mean-pooled features. For completeness, we also evaluate the same frozen encoder using mean-pooled features and an attentive probe, obtaining 43.4 and 62.0, respectively. The lower mean-pooled result is consistent with these features not being directly regularized in our setup.

\textbf{Takeaway.} $\lambda$-JEPA extends effectively to video SSL, with its clearest gains on temporal evaluations, showing that the benefits of backbone spectral regularization transfer beyond image pretraining.

\begin{figure*}[t]
    \centering
    \includegraphics[width=0.8\textwidth]{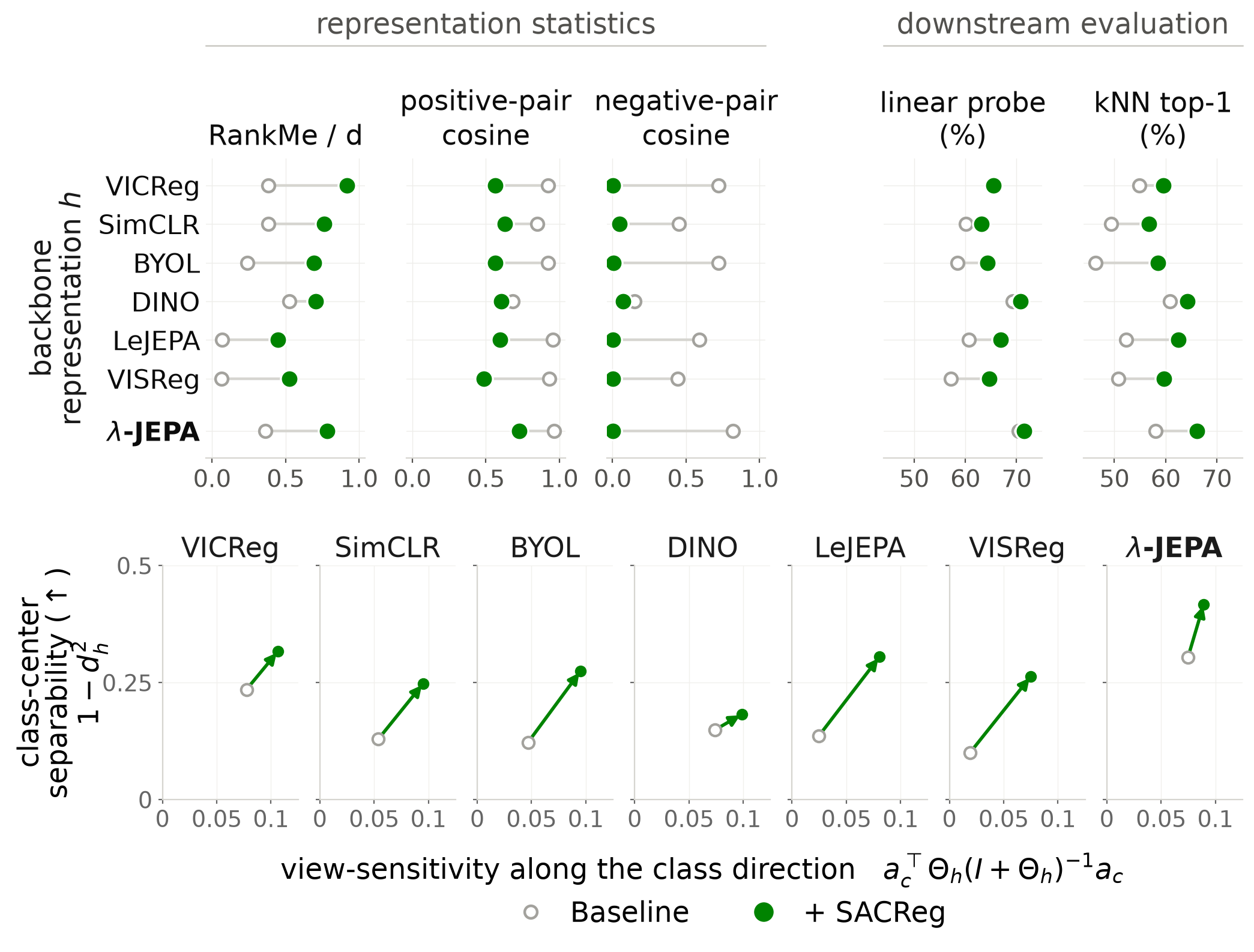}
    \caption{
    Effect of backbone SACReg across SSL objectives on ImageNet-100.
    \textbf{Top:} backbone statistics and downstream evaluation.
    \textbf{Bottom:} class-center separability versus view sensitivity along class-prediction directions.
    Open and green markers denote baseline and +SACReg, respectively; arrows show their change.
    Detailed results are in Appendix~\ref{app:detailed-results}.
    }
    \label{fig:in100-audit-org}
\end{figure*}

\subsection{Controlled experiments on ImageNet-100}
\label{sec:exp-in100}
We next use controlled ImageNet-100 experiments to better understand the effects of SACReg across SSL objectives. Our analysis focuses on two questions: how SACReg changes representation geometry and downstream performance, and what variation is recovered in the additional representation directions. To this end, we apply SACReg to the backbone representations of six existing augmentation-based SSL methods on ImageNet-100~\citep{inet100tian2020contrastive}: LeJEPA~\citep{lejepa2025}, VICReg~\citep{vicreg2022}, VISReg~\citep{visreg2026}, SimCLR~\citep{simclr2020}, DINO~\citep{dinocaron2021}, and BYOL~\citep{byolgrill2020}. For each method, we keep its original objective and training recipe unchanged and add only $\beta_{\mathrm{h}}\operatorname{SACReg}$ to the backbone representation. The regularizer weight is set using the gradient based calibration procedure (see Appendix~\ref{app:loss-weights} for details) and is not tuned individually to the methods except DINO. For DINO we tune the weight separately due to its sensitivity to the gradient ratios.

Figure~\ref{fig:in100-audit-org} (top) shows that applying SACReg to the backbones produces a consistent change in the representations across objectives. At the backbone, RankMe increases substantially for every method, while the cosine geometry becomes less concentrated. Negative-pair cosine similarities approach zero while positive-pairs generally decrease as well. In contrast, corresponding statistics in the projection/loss space change much less. This suggests that SACReg primarily reshapes the backbone representation without substantially impeding the learning after the projection. The same intervention also improves downstream performance. kNN accuracy improves across all methods convincingly, while linear probe accuracy also improves, with smaller gains for VICReg. Notably, the gains therefore do not come from making positive pairs more invariant. %; instead, SACReg broadens the backbone representation while preserving additional variation that is useful for downstream prediction.

% \textbf{Augmentation thickness}
% We next distinguish variation between images from variation across augmented views of the same image. Let $Q$ denote a source image and $\mathbf{h}$ its backbone representation under a random augmentation. We define $B_h = \operatorname{Cov}\left(\mathbb{E}[h\mid Q]\right)$ and $A_h = \mathbb{E}\left[\operatorname{Cov}(h\mid Q)\right]$, where $B_h$ captures variation between image centers, with $r_h=\operatorname{rank}(B_h)$, and $A_h$ captures variation across views of the same image. We then define the \emph{augmentation thickness}
% \begin{equation}
%     \Theta_h = B_h^{\dagger/2}A_hB_h^{\dagger/2},
%     \qquad
%     \theta_h(a)
%     =
%     \frac{a^\top A_h a}{a^\top B_h a},
%     \label{eq:thickness-main}
% \end{equation}
% so thickness measures augmentation variation relative to the separation of image centers along that direction. Appendix~\ref{app:thickness} shows that thickness has a two-sided role: too little can discard meaningful variation across augmentations, while too much allows within-image variation to dominate the separation between images.

% \looseness=-1To study thickness, we use the class indicators. For each class $c$, let $a_c$ be its linear prediction from whitened image centers and $d_h(c)^2$ the corresponding error. We measure class-center separability by $1-d_h(c)^2$ and view sensitivity by $a_c^\top \Theta_h (I_{r_h} + \Theta_h)^{-1} a_c$. Details are in Appendix~\ref{app:exp-details}.

\textbf{Augmentation thickness}
We next distinguish variation between images from variation across augmented views of the same image. Let $Q$ denote a source image and $\mathbf{h}$ its backbone representation under a random augmentation. We define $\mathbf{B}_h = \operatorname{Cov}\left(\mathbb{E}[\mathbf{h}\mid Q]\right)$ and $\mathbf{A}_h = \mathbb{E}\left[\operatorname{Cov}(\mathbf{h}\mid Q)\right]$, where $\mathbf{B}_h$ captures variation between image centers, with $r_h=\operatorname{rank}(\mathbf{B}_h)$, and $\mathbf{A}_h$ captures variation across views of the same image. We then define the \emph{augmentation thickness}
\begin{equation}
    \boldsymbol{\Theta}_h = \mathbf{B}_h^{\dagger/2}\mathbf{A}_h\mathbf{B}_h^{\dagger/2},
    \qquad
    \theta_h(\mathbf{a})
    =
    \frac{\mathbf{a}^\top \mathbf{A}_h \mathbf{a}}{\mathbf{a}^\top \mathbf{B}_h \mathbf{a}},
    \label{eq:thickness-main}
\end{equation}
so thickness measures augmentation variation relative to the separation of image centers along that direction. Appendix~\ref{app:thickness} shows that thickness has a two-sided role: too little can discard meaningful variation across augmentations, while too much allows within-image variation to dominate the separation between images.

\looseness=-1To study thickness, we use the class indicators. For each class $c$, let $\mathbf{a}_c$ be its linear prediction from whitened image centers and $d_h(c)^2$ the corresponding error. We measure class-center separability by $1-d_h(c)^2$ and view sensitivity by $\mathbf{a}_c^\top \boldsymbol{\Theta}_h (\mathbf{I}_{r_h} + \boldsymbol{\Theta}_h)^{-1} \mathbf{a}_c$. Details are in Appendix~\ref{app:exp-details}.

Figure~\ref{fig:in100-audit-org} (bottom) shows that, averaged across classes, all six SSL methods have greater class-center separability and greater view sensitivity under SACReg. Therefore, SACReg's downstream gains are not obtained by suppressing augmentation-dependent variation. Specifically, the more consistent effect is an improvement in how image centers are organized along downstream-relevant directions.  For $\lambda$-JEPA, class-center separability improves for 96 of 100 classes, whereas view sensitivity moves in both directions. SACReg therefore consistently improves class-center separation without uniformly moving the representation toward either invariance or augmentation sensitivity.

\textbf{Takeaway.} Across diverse SSL objectives, SACReg consistently increases representation rank and improves downstream performance. The augmentation thickness analysis shows that these gains are not explained by increased augmentation invariance. SACReg improves class-center separation while preserving augmentation variation.

\section{Conclusion}
In this paper, we targeted the asymmetry between anti-collapse regularization in joint embedding self-supervised models and the representation used for downstream tasks. First, we showed that preventing collapse in the projected representation does not guarantee a high rank backbone representation. This matters because a low rank backbone reduces the representational capacity, which can harm downstream performance. Our solution is directly motivated by a linear analysis of negative layer balance, which is known as $\lambda$ balance in the feature learning literature. From this, we introduced SACReg, which directly regularizes the representation covariance, and used it to build $\lambda$-JEPA. In controlled experiments, SACReg consistently increases backbone rank and improves downstream performance while preserving augmentation-dependent variation. At scale, $\lambda$-JEPA narrows the gap between explicitly regularized JEPA methods and DINO, while improving frozen transfer across image and video SSL benchmarks.

\looseness=-1SACReg regularizes CLS tokens and uses global views. Extending it to patch-level features and local crops is non-trivial: local views introduce substantially larger augmentation-induced variation, making the per-image view centers noisier and the covariance across centers harder to regularize reliably. In our experiments, this led to unstable training. Understanding how different augmentations shape this variation and which augmentation-sensitive directions are useful for downstream tasks remains an open direction for future work. Developing a formulation that more robustly regularizes spatially local representations while separating augmentation variation from between-image variation could further improve performance on dense prediction tasks. Additionally, JEPA-style models have already been connected with world modeling and identifiabilty theory for learning causal representations~\citep{klindt2026does}. This is conceptually natural, as JEPA models optimize the same principles of causal representations: invariance to relevant transformations and anti-collapse~\citep{yao2025unifying}. However, methods in causal representation learning leverage mostly negative samples (and sometimes reconstruction) to implement the anti-collapse properties. We encourage future work to explore the role of our regularizer in the identifiability of world models. %and further connect with the lazy-and-rich learning literature. 

% Good?
% Future works may explor and so their synergy with JEPA and the regularizer discussed in this work remains an open question.%  to be investigated.

\subsection*{AI use statement}
We used generative AI tools as coding assistants to support the implementation and modification of the codebase and for language editing. These tools were not used to autonomously generate research results or replace the authors' scientific judgment. The authors determined the research methodology and experimental design, reviewed and validated all AI-assisted implementations.

\subsection*{Ethics statement}
This work studies methods for self-supervised representation learning using established image and video datasets and does not involve human-subject experiments or the collection of new personal or sensitive data. We are not aware of specific ethical concerns introduced by the proposed method beyond those generally associated with large-scale representation learning and computer vision models. The work is intended as a general-purpose methodological contribution and we encourage appropriate consideration of dataset provenance, bias and privacy risks when deploying learned representations.

\subsection*{Reproducibility statement}
We provide detailed information required to reproduce our theoretical and empirical results throughout the paper and appendices. The assumptions, statements and proofs supporting the theoretical results are provided in Appendix~\ref{app:theory}. For the empirical results, implementation details for SACReg are given in Appendix~\ref{app:slicing}, Appendix~\ref{app:exp-details} includes information about datasets, architectures, training procedures, loss-weight calibration, evaluation pipelines and baselines used in our experiments. Additional analyses and their definitions are provided in Appendices~\ref{app:lejepa} and~\ref{app:thickness}. We reuse public baseline implementations and checkpoints whenever available, and explicitly document differences in training and evaluation protocols when exact matching is not possible. In particular, for video SSL, some baselines use different frozen-feature evaluation protocols (e.g., CLS, mean-pooled, or attentive probes); we report the corresponding protocols and evaluate our representations under multiple probe choices to facilitate comparison. Furthermore, we provide supplementary material containing the code used to run every experiment reported in the paper. These materials are intended to enable reproduction of both the proposed method and the empirical evaluations.

\subsection*{Acknowledgments}
This research was funded in whole or in part by the Austrian Science Fund (FWF) [10.55776/COE12]. CD thank Samuel Lippl and Nicolas Anguita for their constructive discussions.
This research was supported by the ISTA Responsible AI Program, made possible through the generous support of Garrett Camp and the Camp Foundation.  M.F. was supported by the project “Building Energy Systems on causal reasoning (BOSS)“, funded within the “Technologies and Innovations for the Climate-Neutral City” (TIKS) Programme
of the Austrian Research Promotion Agency (FFG).

\bibliography{iclr2027_conference}
\bibliographystyle{iclr2027_conference}

\appendix

\section{A non-collapse regularizer inspired by the feature-learning literature}
\label{app:theory}
\label{sec:exact_learning_dynamics}

\subsection{Linear network}\label{app:linnet}

\subsubsection{Problem setup}

Consider a supervised learning problem with training data
$
\left\{
    \left(\mathbf{x}_n,\mathbf{y}_n\right)
\right\}_{n=1}^{P},
\mathbf{x}_n\in\mathbb{R}^{N_i},
\mathbf{y}_n\in\mathbb{R}^{N_o}.
$
Let
$
\mathbf{X}=
\begin{bmatrix}
\mathbf{x}_1 & \cdots & \mathbf{x}_P
\end{bmatrix}
\in\mathbb{R}^{N_i\times P},$ and $
\mathbf{Y}=
\begin{bmatrix}
\mathbf{y}_1 & \cdots & \mathbf{y}_P
\end{bmatrix}
\in\mathbb{R}^{N_o\times P}.
$
We consider a two-layer linear network
$ \widehat{\mathbf{y}}_n = \wb\wa\mathbf{x}_n,
$
with encoder and decoder matrices
$
\wa\in\mathbb{R}^{N_h\times N_i}$,
$\wb\in\mathbb{R}^{N_o\times N_h}.
$
The corresponding end-to-end mapping is
$
\mathbf{M}=\wb\wa.
$

The empirical mean-squared error is
\begin{equation}
\mathcal{L}_{\mathrm{task}}(\wa,\wb)
=
\frac{1}{2P}
\left\|
    \wb\wa \mathbf{X}-\mathbf{Y}
\right\|_F^2.
\label{eq:task-loss}
\end{equation}
%More generally, the results below apply to any differentiable task loss of the form $\mathcal{L}_{\mathrm{task}}(\wa,\wb)=\ell(\wb\wa)$
%that depends on the parameter\\lambdas only through the end-to-end mapping. 
Throughout this section, we consider continuous-time gradient flow,
with time constant \(\tau>0\):
\begin{equation}
\tau\dot{\wa}
=
-\nabla_{\wa}\mathcal{L}_{\mathrm{task}},
\qquad
\tau\dot{\wb}
=
-\nabla_{\wb}\mathcal{L}_{\mathrm{task}}.
\label{eq:gradient-flow-definition}
\end{equation}

\begin{definition}[\(\lambda\)-balancedness]
\label{def:lambda-balanced}
Define the layer imbalance matrix by
\begin{equation}
\mathbf{\Delta}(\wa,\wb)
=
\wb^\top\wb-\wa\wa^\top
\in\mathbb{R}^{N_h\times N_h}.
\label{eq:imbalance-definition}
\end{equation}
The network is said to be \(\lambda_{\mathrm{bal}}\)-balanced if
$\mathbf{\Delta}(\wa,\wb)=\lambda_{\mathrm{bal}}\mathbf{I}_{N_h}.$
The case \(\lambda_{\mathrm{bal}}=0\) corresponds to standard
balancedness.
\end{definition}

\subsubsection{Balance conservation without explicit regularization}

For completeness we first recall the balance-conservation property of two-layer linear
networks trained without explicit regularization, also derived in \cite{domine2025lazy}.

\begin{lemma}[Conservation of layer imbalance]
\label{lem:unregularised-balance-conservation}
Suppose that the network is trained by gradient flow on a
differentiable loss of the form
$ \mathcal{L}_{\mathrm{task}}(\wa,\wb) = \ell(\wb\wa),$
without weight decay or a log-determinant regularizer. Then
\begin{equation}
\frac{\mathrm{d}}{\mathrm{d}t}
\left(
    \wb^\top\wb-\wa\wa^\top
\right)
=
\mathbf{0}.
\label{eq:unregularised-balance-conservation}
\end{equation}
Consequently, if the network is
\(\lambda_{\mathrm{bal}}\)-balanced at initialization, then it remains
\(\lambda_{\mathrm{bal}}\)-balanced throughout gradient-flow training.
\end{lemma}

\begin{proof}
Let
$
\mathbf{G}
=
\nabla_{\mathbf{M}}\ell(\mathbf{M})
\big|_{\mathbf{M}=\mathbf{W}_2\mathbf{W}_1}.
$
The gradient-flow equations for the unregularized task loss are
\[
\tau\dot{\wa}
=
-\wb^\top \mathbf{G},
\qquad
\tau\dot{\wb}
=
-\mathbf{G}\wa^\top.
\]
Therefore,
\begin{align}
\tau\frac{\mathrm{d}}{\mathrm{d}t}
\left(\wb^\top\wb\right)
&=
-\wa \mathbf{G}^\top\wb
-\wb^\top \mathbf{G}\wa^\top,
\\
\tau\frac{\mathrm{d}}{\mathrm{d}t}
\left(\wa\wa^\top\right)
&=
-\wb^\top \mathbf{G}\wa^\top
-\wa \mathbf{G}^\top\wb.
\end{align}
The two expressions are identical. Subtracting them proves
\eqref{eq:unregularised-balance-conservation}.
\end{proof}

\begin{remark}[Gradient flow versus gradient descent]
The conservation law in
Lemma~\ref{lem:unregularised-balance-conservation} is exact under
continuous-time gradient flow. For finite-step gradient descent,
additional terms of order \(\eta^2\) generally appear, so exact
conservation need not hold for a nonzero learning rate \(\eta\).
\end{remark}

\subsubsection{Non-collapse property of balanced initialization}\label{app:noncollapse}

We first show that a negative layer balance is sufficient to prevent
encoder collapse. Under a positive-definiteness assumption on the input
covariance, it also prevents collapse of the hidden representations.
\begin{lemma}[Non-collapse under negative layer balance]
\label{lem:balanced-non-collapse}
Let $\wa\in\mathbb{R}^{N_h\times N_i}$ and
$\wb\in\mathbb{R}^{N_o\times N_h}$ satisfy
\begin{equation}
\wb^\top\wb-\wa\wa^\top
=
\lambda_{\mathrm{bal}}\mathbf{I}_{N_h},
\qquad \lambda_{\mathrm{bal}}<0.
\label{eq:negative-layer-balance}
\end{equation}
Then
$
\wa\wa^\top
\succeq
-\lambda_{\mathrm{bal}}\mathbf{I}_{N_h}.$
Consequently, $\wa$ has full row rank, with
$\sigma_{\min}(\wa)\geq\sqrt{-\lambda_{\mathrm{bal}}}$.

Furthermore, let
$\bar{\mathbf{x}}=P^{-1}\sum_{n=1}^{P}\mathbf{x}_n$
and suppose that the centered empirical input covariance satisfies $\boldsymbol{\Sigma}_x
=
\frac{1}{P}\sum_{n=1}^{P}
(\mathbf{x}_n-\bar{\mathbf{x}})
(\mathbf{x}_n-\bar{\mathbf{x}})^\top
\succeq
\kappa_x \mathbf{I}_{N_i}
\text{ for some }\kappa_x >0.$
For the hidden representations $\mathbf{h}_n=\wa\mathbf{x}_n$,
let $\bar{\mathbf{h}}=P^{-1}\sum_{n=1}^{P}\mathbf{h}_n$.
Their centered empirical covariance then satisfies
\begin{equation}
\boxed{
\mathbf{C}_h
=
\frac{1}{P}\sum_{n=1}^{P}
(\mathbf{h}_n-\bar{\mathbf{h}})
(\mathbf{h}_n-\bar{\mathbf{h}})^\top
\succeq
-\kappa_x \lambda_{\mathrm{bal}}\mathbf{I}_{N_h}
}.
\label{eq:balanced-representation-non-collapse}
\end{equation}

\end{lemma}

\begin{proof}
Rearranging \eqref{eq:negative-layer-balance} gives
\[
\wa\wa^\top
=
\wb^\top\wb-\lambda_{\mathrm{bal}}\mathbf{I}_{N_h}
\succeq
-\lambda_{\mathrm{bal}}\mathbf{I}_{N_h}
\succ0,
\]
since $\wb^\top\wb\succeq0$ and $\lambda_{\mathrm{bal}}<0$.
This establishes the encoder rank and singular-value bounds.

By linearity, $\bar{\mathbf{h}}=\wa\bar{\mathbf{x}}$, so
$\mathbf{h}_n-\bar{\mathbf{h}}
=\wa(\mathbf{x}_n-\bar{\mathbf{x}})$.
Therefore,
\[
\mathbf{C}_h
= \mathbf{W}_1\boldsymbol{\Sigma}_x\mathbf{W}_1^\top
\succeq \kappa_x\mathbf{W}_1\mathbf{W}_1^\top
\succeq -\kappa_x\lambda_{\mathrm{bal}}\mathbf{I}_{N_h}
\succ 0,
\]
which proves the representation non-collapse bound.
\end{proof}
Thus, the representation covariance is positive definite,
with variance at least $-\kappa_x\lambda_{\mathrm{bal}}>0$
in every unit direction.
This rules out \emph{complete collapse}, in which all hidden
representations coincide ($\mathbf{C}_h=\mathbf{0}$), and, more generally,
\emph{dimensional collapse}, in which the hidden representations lie
in a proper affine subspace of $\mathbb{R}^{N_h}$
($\operatorname{rank}(\mathbf{C}_h)<N_h$).

%\begin{remark}
%The representation guarantee requires neither whitened inputs nor
%zero-mean training data. \mm{this is the type of comments an LLM would make, scratch it or rephrase it}  It requires a strictly positive lower bound
%on the centered input covariance. Without this condition, negative
%layer balance still guarantees that the encoder has full row rank,
%but does not by itself guarantee a full-rank representation covariance.
%\end{remark}

Combining Lemma~\ref{lem:unregularised-balance-conservation} with
Lemma~\ref{lem:balanced-non-collapse}, we conclude that a network
initialized with a negative layer balance retains the corresponding
non-collapse bounds throughout unregularized gradient-flow training.
This guarantee depends on the balance imposed at initialization.

\subsubsection{Constrained optimization characterization}
\label{app:opptimisation}
Lemma~\ref{lem:balanced-non-collapse} identifies negative layer balance
as a sufficient condition for non-collapse. We next show that the
proposed regularizer selects precisely such a balance among factorizations of a fixed end-to-end mapping. The regularizer arises from a constrained minimum-regularization problem.

\begin{theorem}[Balancedness via Constrained Optimization ]
\label{app:thm:encoder-variational-balance}
Fix an end-to-end mapping $\mathbf{M}\in\mathbb R^{N_o\times N_i}$, $\gamma > 0$, $\lambda_{\mathrm{reg}}> 0$
and consider
\begin{equation}
\begin{aligned}
\min_{\wa,\wb}\quad&
\frac{\lambda_{\mathrm{reg}}}{2}
\left(
    \|\wa\|_F^2+\|\wb\|_F^2
\right)
-
\frac{\gamma}{2}
\log\det\left(\wa\wa^\top\right)
\\
\text{subject to}\quad&
\wb\wa=\mathbf{M}.
\end{aligned}
\end{equation}
Let \((\wa^\star,\wb^\star)\) be a constrained minimizer with
\(\wa^\star(\wa^\star)^\top\succ0\). Then
\begin{equation}
\boxed{
(\wb^\star)^\top\wb^\star
-
\wa^\star(\wa^\star)^\top
=
-\frac{\gamma}{\lambda_{\mathrm{reg}}}
\mathbf{I}_{N_h}
}.
\label{eq:encoder-variational-balance}
\end{equation}
\end{theorem}

\begin{proof}
Introduce a Lagrange multiplier
$\mathbf{\Lambda}\in\mathbb{R}^{N_o\times N_i}$
for the constraint \(\wb\wa=\mathbf{M}\). The Lagrangian is
\begin{equation}
\begin{aligned}
\mathscr{L}(\wa,\wb,\mathbf{\Lambda})
={}&
\frac{\lambda_{\mathrm{reg}}}{2}
\left(
    \|\wa\|_F^2+\|\wb\|_F^2
\right)
-
\frac{\gamma}{2}
\log\det\left(\wa\wa^\top\right)
\\
&+
\left\langle
    \mathbf{\Lambda},\wb\wa-\mathbf{M}
\right\rangle_F.
\end{aligned}
\label{eq:encoder-constrained-lagrangian}
\end{equation}
The first-order stationarity conditions are
\begin{align}
\lambda_{\mathrm{reg}}\wa
-
\gamma
\left(\wa\wa^\top\right)^{-1}\wa
+
\wb^\top\mathbf{\Lambda}
&=\mathbf{0},
\label{eq:encoder-kkt-wa}
\\
\lambda_{\mathrm{reg}}\wb
+
\mathbf{\Lambda}\wa^\top
&=\mathbf{0}.
\label{eq:encoder-kkt-wb}
\end{align}
Multiplying \eqref{eq:encoder-kkt-wa} on the right by \(\wa^\top\)
gives
\begin{equation}
\lambda_{\mathrm{reg}}\wa\wa^\top
-
\gamma\mathbf{I}_{N_h}
+
\wb^\top\mathbf{\Lambda}\wa^\top
=\mathbf{0}.
\label{eq:encoder-kkt-gram-wa}
\end{equation}
Multiplying \eqref{eq:encoder-kkt-wb} on the left by \(\wb^\top\)
gives
\begin{equation}
\lambda_{\mathrm{reg}}\wb^\top\wb
+
\wb^\top\mathbf{\Lambda}\wa^\top
=\mathbf{0}.
\label{eq:encoder-kkt-gram-wb}
\end{equation}
Eliminating the common multiplier term
\(\wb^\top\mathbf{\Lambda}\wa^\top\) yields
$
\lambda_{\mathrm{reg}}
\left(
    \wa\wa^\top-\wb^\top\wb
\right)
=
\gamma\mathbf{I}_{N_h},
$
which is equivalent to
\eqref{eq:encoder-variational-balance}.
\end{proof}

\begin{corollary}[Non-collapse at constrained regularizer minima]
\label{cor:encoder-non-collapse}
Under the assumptions of
Theorem~\ref{app:thm:encoder-variational-balance},
\begin{equation}
\boxed{
\wa^\star(\wa^\star)^\top
\succeq
\frac{\gamma}{\lambda_{\mathrm{reg}}}
\mathbf{I}_{N_h}
}.
\label{eq:encoder-anti-collapse-bound}
\end{equation}
Consequently, the encoder has full row rank, with
$\sigma_{\min}(\wa^\star)
\geq
\sqrt{\frac{\gamma}{\lambda_{\mathrm{reg}}}}.$

If, in addition, the centered empirical input covariance $\boldsymbol{\Sigma}_x$
defined in Lemma~\ref{lem:balanced-non-collapse} satisfies
$\boldsymbol{\Sigma}_x \succeq \kappa_x\mathbf{I}_{N_i}$ for some $\kappa_x >0$,
then the centered empirical covariance of the hidden representations
$\mathbf{h}_n^\star=\wa^\star\mathbf{x}_n$ satisfies
\begin{equation}
\boxed{
\mathbf{C}_h^\star
=
\wa^\star\boldsymbol{\Sigma}_x(\wa^\star)^\top
\succeq
\kappa_x \frac{\gamma}{\lambda_{\mathrm{reg}}}
\mathbf{I}_{N_h}
}.
\label{eq:representation-anti-collapse-bound}
\end{equation}
\end{corollary}

\begin{proof}
Theorem~\ref{app:thm:encoder-variational-balance} gives
\[
(\wb^\star)^\top\wb^\star
-
\wa^\star(\wa^\star)^\top
=
-\frac{\gamma}{\lambda_{\mathrm{reg}}}
\mathbf{I}_{N_h}.
\]
Since $\gamma>0$ and $\lambda_{\mathrm{reg}}>0$, this is a
negative layer balance. Both bounds therefore follow directly from
Lemma~\ref{lem:balanced-non-collapse} with
$\lambda_{\mathrm{bal}}=-\gamma/\lambda_{\mathrm{reg}}<0$.  
\end{proof}

Thus, the representation covariance is positive definite, preventing
both complete and dimensional collapse.

\subsubsection{Balance selected by the encoder-side regularizer}\label{app:balsel}

For arbitrary $\mathbf{\Delta}$, the regularizer exponentially drives the network toward the selected balance.

\begin{theorem}[Global existence and exponential convergence of layer balance]
\label{thm:encoder-balance-dynamics}
Consider gradient flow with time scale $\tau>0$ on the regularized
mean-squared-error objective in
\eqref{eq:encoder-regularised-objective}, with
$\lambda_{\mathrm{reg}}>0$ and $\gamma>0$.
If $\wa(0)\wa(0)^\top\succ0$, then the solution exists for all
$t\geq0$ and the encoder remains full row rank.

Define
\[
\mathbf{\Delta}(t)=\wb(t)^\top\wb(t)-\wa(t)\wa(t)^\top,
\qquad
\mathbf{\Delta}_\star
=
-\frac{\gamma}{\lambda_{\mathrm{reg}}}\mathbf{I}_{N_h}.
\]
The imbalance satisfies
\[
\tau\dot{\mathbf{\Delta}}(t)
=
-2\lambda_{\mathrm{reg}}\bigl(\mathbf{\Delta}(t)-\mathbf{\Delta}_\star\bigr),
\]
and hence
\begin{equation}
\boxed{
\mathbf{\Delta}(t)-\mathbf{\Delta}_\star
=
e^{-2\lambda_{\mathrm{reg}}t/\tau}
\bigl(\mathbf{\Delta}(0)-\mathbf{\Delta}_\star\bigr)
}.
\label{eq:encoder-balance-solution} 
\end{equation}
Thus, for every admissible initialization, the layer imbalance
converges exponentially to $\mathbf{\Delta}_\star$ at rate
$2\lambda_{\mathrm{reg}}/\tau$, selecting the asymptotic balance
$\lambda_{\mathrm{bal}}^{\mathrm{enc}}
=-\gamma/\lambda_{\mathrm{reg}}$.
\end{theorem}

\begin{proof}
We first establish global existence and preservation of the encoder
full row rank. Along gradient flow,
\[
\frac{\mathrm{d}}{\mathrm{d}t}\mathcal{L}_{\mathrm{LSAC}}
=
-\frac{1}{\tau}
\left(
\|\nabla_{\wa}\mathcal{L}_{\mathrm{LSAC}}\|_F^2
+
\|\nabla_{\wb}\mathcal{L}_{\mathrm{LSAC}}\|_F^2
\right)
\leq0.
\]
Writing $\sigma_i$ for the encoder singular values, the regularizer is
\[
\mathcal{R}_{\mathrm{LSAC}}
=
\frac{\lambda_{\mathrm{reg}}}{2}\|\wb\|_F^2
+
\sum_{i=1}^{N_h}
\left(
\frac{\lambda_{\mathrm{reg}}}{2}\sigma_i^2
-\gamma\log\sigma_i
\right).
\]
Each scalar summand is bounded below and diverges at both zero
and infinity. Since the mean-squared-error task loss is nonnegative,
the initial objective sublevel set bounds the weights and keeps
every encoder singular value away from zero. The trajectory thus
remains in a compact subset of the full-row-rank domain, where the
gradient field is smooth. Standard ODE continuation gives global
existence and $\wa(t)\wa(t)^\top\succ0$ for all $t\geq0$.

Now define
\[
\mathbf{G}=\nabla_\mathbf{M}\ell(\mathbf{M})\big|_{\mathbf{M}=\wb\wa},
\qquad
\mathbf{H}=\wa\wa^\top.
\]
Since $\mathbf{H}\succ0$, it is symmetric and invertible. The gradient of the
encoder-side log-determinant term is
\begin{equation}
\nabla_{\wa}
\left[
-\frac{\gamma}{2}
\log\det\left(\wa\wa^\top\right)
\right]
=
-\gamma \mathbf{H}^{-1}\wa.
\label{eq:encoder-logdet-gradient}
\end{equation}
The gradient-flow equations are therefore
\begin{align}
\tau\dot{\wa}
&=
-\wb^\top \mathbf{G}
-\lambda_{\mathrm{reg}}\wa
+\gamma \mathbf{H}^{-1}\wa,
\label{eq:encoder-wa-flow}
\\
\tau\dot{\wb}
&=
-\mathbf{G}\wa^\top
-\lambda_{\mathrm{reg}}\wb.
\label{eq:encoder-wb-flow}
\end{align}

We first compute the dynamics of the decoder Gram matrix.
By the product rule,
\[
\tau\frac{\mathrm{d}}{\mathrm{d}t}
\left(\wb^\top\wb\right)
=
\left(\tau\dot{\wb}\right)^\top\wb
+
\wb^\top\left(\tau\dot{\wb}\right).
\]
Substituting \eqref{eq:encoder-wb-flow} gives
\begin{align}
\tau\frac{\mathrm{d}}{\mathrm{d}t}
\left(\wb^\top\wb\right)
&=
\left(
-\mathbf{G}\wa^\top-\lambda_{\mathrm{reg}}\wb
\right)^\top\wb
+
\wb^\top
\left(
-\mathbf{G}\wa^\top-\lambda_{\mathrm{reg}}\wb
\right)
\nonumber\\
&=
-\wa \mathbf{G}^\top\wb
-\wb^\top \mathbf{G}\wa^\top
-2\lambda_{\mathrm{reg}}\wb^\top\wb.
\label{eq:decoder-gram-dynamics}
\end{align}

Similarly, the encoder Gram matrix satisfies
\[
\tau\frac{\mathrm{d}}{\mathrm{d}t}
\left(\wa\wa^\top\right)
=
\left(\tau\dot{\wa}\right)\wa^\top
+
\wa\left(\tau\dot{\wa}\right)^\top.
\]
Substituting \eqref{eq:encoder-wa-flow}, we obtain
\begin{align}
\tau\frac{\mathrm{d}}{\mathrm{d}t}
\left(\wa\wa^\top\right)
={}&
\left(
-\wb^\top \mathbf{G}
-\lambda_{\mathrm{reg}}\wa
+\gamma \mathbf{H}^{-1}\wa
\right)\wa^\top
\nonumber\\
&+
\wa
\left(
-\wb^\top \mathbf{G}
-\lambda_{\mathrm{reg}}\wa
+\gamma \mathbf{H}^{-1}\wa
\right)^\top
\nonumber\\
={}&
-\wb^\top \mathbf{G}\wa^\top
-\wa \mathbf{G}^\top\wb
-2\lambda_{\mathrm{reg}}\wa\wa^\top
\nonumber\\
&+
\gamma \mathbf{H}^{-1}\wa\wa^\top
+
\gamma\wa\wa^\top \mathbf{H}^{-1}.
\end{align}
Because $\mathbf{H}=\wa\wa^\top$, we have
\[
\mathbf{H}^{-1}\wa\wa^\top=\mathbf{I}_{N_h},
\qquad
\wa\wa^\top \mathbf{H}^{-1}=\mathbf{I}_{N_h}.
\]
Therefore,
\begin{equation}
\tau\frac{\mathrm{d}}{\mathrm{d}t}
\left(\wa\wa^\top\right)
=
-\wb^\top \mathbf{G}\wa^\top
-\wa \mathbf{G}^\top\wb
-2\lambda_{\mathrm{reg}}\wa\wa^\top
+2\gamma\mathbf{I}_{N_h}.
\label{eq:encoder-gram-dynamics}
\end{equation}

Recall that
\[
\mathbf{\Delta}=\wb^\top\wb-\wa\wa^\top.
\]
Subtracting \eqref{eq:encoder-gram-dynamics} from
\eqref{eq:decoder-gram-dynamics} yields
\begin{align}
\tau\dot{\mathbf{\Delta}}
={}&
\left(
-\wa \mathbf{G}^\top\wb
-\wb^\top \mathbf{G}\wa^\top
-2\lambda_{\mathrm{reg}}\wb^\top\wb
\right)
\nonumber\\
&-
\left(
-\wb^\top \mathbf{G}\wa^\top
-\wa \mathbf{G}^\top\wb
-2\lambda_{\mathrm{reg}}\wa\wa^\top
+2\gamma\mathbf{I}_{N_h}
\right)
\nonumber\\
={}&
-2\lambda_{\mathrm{reg}}
\left(
\wb^\top\wb-\wa\wa^\top
\right)
-2\gamma\mathbf{I}_{N_h}.
\end{align}
The task-loss terms cancel exactly, leaving
\[
\boxed{
\tau\dot{\mathbf{\Delta}}
=
-2\lambda_{\mathrm{reg}}\mathbf{\Delta}
-2\gamma\mathbf{I}_{N_h}
}.
\]

It remains to solve this linear matrix differential equation.
Dividing by $\tau$ gives
\[
\dot{\mathbf{\Delta}}
+
\frac{2\lambda_{\mathrm{reg}}}{\tau}\mathbf{\Delta}
=
-\frac{2\gamma}{\tau}\mathbf{I}_{N_h}.
\]
Multiplying both sides by the integrating factor
$e^{2\lambda_{\mathrm{reg}}t/\tau}$ yields
\[
\frac{\mathrm{d}}{\mathrm{d}t}
\left[
e^{2\lambda_{\mathrm{reg}}t/\tau}\mathbf{\Delta}(t)
\right]
=
-\frac{2\gamma}{\tau}
e^{2\lambda_{\mathrm{reg}}t/\tau}
\mathbf{I}_{N_h}.
\]
Integrating from $0$ to $t$, we obtain
\begin{align*}
e^{2\lambda_{\mathrm{reg}}t/\tau}\mathbf{\Delta}(t)-\mathbf{\Delta}(0)
&=
-\frac{2\gamma}{\tau}
\int_0^t
e^{2\lambda_{\mathrm{reg}}s/\tau}
\,\mathrm{d}s\,
\mathbf{I}_{N_h}
\\
&=
-\frac{\gamma}{\lambda_{\mathrm{reg}}}
\left(
e^{2\lambda_{\mathrm{reg}}t/\tau}-1
\right)
\mathbf{I}_{N_h}.
\end{align*}
Multiplying by $e^{-2\lambda_{\mathrm{reg}}t/\tau}$ gives
\[
\mathbf{\Delta}(t)
=
e^{-2\lambda_{\mathrm{reg}}t/\tau}\mathbf{\Delta}(0)
-
\frac{\gamma}{\lambda_{\mathrm{reg}}}
\left(
1-e^{-2\lambda_{\mathrm{reg}}t/\tau}
\right)
\mathbf{I}_{N_h}.
\]
Equivalently, defining
$\mathbf{\Delta}_\star=-(\gamma/\lambda_{\mathrm{reg}})\mathbf{I}_{N_h}$,
\[
\boxed{
\mathbf{\Delta}(t)-\mathbf{\Delta}_\star
=
e^{-2\lambda_{\mathrm{reg}}t/\tau}
\bigl(\mathbf{\Delta}(0)-\mathbf{\Delta}_\star\bigr)
}.
\]
Since $\lambda_{\mathrm{reg}}>0$ and $\tau>0$, the imbalance
converges exponentially to $\mathbf{\Delta}_\star$ at rate
$2\lambda_{\mathrm{reg}}/\tau$.
%No balancedness assumption on $\mathbf{\Delta}(0)$ is required \mm{this is the type of comments an LLM would make, scratch it}.
\end{proof}

\begin{corollary}[Persistence of the selected balance]
\label{cor:selected-balance-persistence}
If the network is initialized such that
$
\mathbf{\Delta}(0)
=
-\frac{\gamma}{\lambda_{\mathrm{reg}}}
\mathbf{I}_{N_h},
$
then
$
\mathbf{\Delta}(t)
=
-\frac{\gamma}{\lambda_{\mathrm{reg}}}
\mathbf{I}_{N_h}
$
for all \(t\geq0\). Thus, the balance condition selected by the
regularizer is invariant under the regularized gradient-flow dynamics.
\end{corollary}

\begin{proof}
Substituting the assumed initial condition into
\eqref{eq:encoder-balance-solution} gives the result immediately.
\end{proof}
% Requires \usepackage{amsmath,amssymb,amsthm}

\subsubsection{Singular-mode characterization}
\label{app:singular}
The previous results establish the selected balance and the anti-collapse
property without requiring an explicit factorization. For completeness,
we now give the corresponding allocation of each end-to-end singular
mode across the two layers, also derived in \cite{domine2025lazy}.  
We later show that the regularization leads to the balanced factorization (see Appendix \ref{app:rich_lazy}).

\medskip

\begin{theorem}[Singular-mode allocation]
\label{thm:singular}
Assume \(N_h \leq \min\{N_i,N_o\}\), and let
\(\mathbf{M} = \mathbf{USV}^\top\), where $\mathbf{S} = \operatorname{diag}(s_1,\ldots,s_{N_h}),
$with $ s_i \geq 0,$
defines an \(N_h\)-mode singular value decomposition, padded with zero
singular values when \(\operatorname{rank}(\mathbf{M}) < N_h\).
Define $\lambda_{bal} = -\frac{\gamma}{\lambda_{\mathrm{reg}}}.$
Up to an orthogonal transformation
\(\mathbf{R} \in \mathbb{R}^{N_h \times N_h}\) of the hidden layer and rotations
within degenerate singular subspaces, a  minimizer of Theorem \ref{app:thm:encoder-variational-balance}
can be written as $\mathbf{W}_1^\star = \mathbf{RAV}^\top, $ $
\mathbf{W}_2^\star = \mathbf{UBR}^\top,
$ 
where $\mathbf{A} = \operatorname{diag}(a_1,\ldots,a_{N_h}),
$ $
\mathbf{B} = \operatorname{diag}(b_1,\ldots,b_{N_h}),$ 
and
\begin{align}
a_i^2
&= \frac{-\lambda_{bal} + \sqrt{\lambda_{bal}^2 + 4s_i^2}}{2},
\qquad
b_i^2 = \frac{\lambda_{bal} + \sqrt{\lambda_{bal}^2 + 4s_i^2}}{2}.
\label{eq:singular-mode-b}
\end{align}
In particular,
$ a_i b_i = s_i,
$ and $
b_i^2 - a_i^2 = \lambda_{bal}.$
\end{theorem}
\begin{proof}
According to Eq. \ref{eq:encoder-variational-balance} $
(\mathbf{W}_2^\star)^\top  \mathbf{W}_2^\star-  \mathbf{W}_1^\star   (\mathbf{W}_1^\star)^\top=
\lambda_{bal}\mathbf{I}_{N_h}.$
The two Gram matrices therefore commute and share an orthogonal
eigenbasis \(R\). Consequently, the factors admit the aligned
representation with $\mathbf{B}^2 -  \mathbf{A}^2 = \lambda_{bal} \mathbf{I}_{N_h}.$

The product constraint gives
\[
\mathbf{W}_2^\star \mathbf{W}_1^\star
=
\mathbf{UBAV}^\top
=
\mathbf{USV}^\top,
\]
and hence \(\mathbf{BA = S}\). For each singular direction \(i\), we therefore have
\[
a_i b_i = s_i,
\qquad
b_i^2 - a_i^2 = \lambda_{bal}.
\]
Setting \(x_i = b_i^2\) gives $
x_i(x_i-\lambda_{bal}) = s_i^2.$ 
The solution is
\[
b_i^2 = \frac{\lambda_{bal} + \sqrt{\lambda_{bal}^2 + 4s_i^2}}{2},
\]
and subtracting \(\lambda_{bal}\) gives
\[
a_i^2 = \frac{-\lambda_{bal} + \sqrt{\lambda_{bal}^2 + 4s_i^2}}{2}.
\]
\end{proof}

\subsubsection{Rich and lazy regime}
\label{app:rich_lazy}

The ratio $-\gamma/\lambda_{\mathrm{reg}}$ determines the selected layer balance and thereby influences the learning regime. For a fixed end-to-end mapping, it governs how each singular mode is distributed between the encoder and decoder. Zero balance ($\lambda_{\mathrm{bal}}=0$) corresponds to equal layer singular values, whereas nonzero balance introduces an asymmetry. The effect of this balance on learning depends on the architecture, initialization scale, and sign of $\lambda_{\mathrm{bal}}$~\citep{domine2025lazy}. In the linear setting considered here, sufficiently strong negative imbalance can favor lazy dynamics while preserving the anti-collapse guarantee. By contrast, zero-balanced initialization can support rich dynamics at sufficiently small initialization scales. The relationship between imbalance and learning regime is more complex in nonlinear networks, where full matrix imbalance is not conserved and does not necessarily induce lazy learning. Appropriate asymmetric parameter scalings can instead accelerate the onset of rich feature-learning dynamics while preserving the non-collapse property~\citep{anguita2026theory,kumar2023grokking,kunin2024get}. We discuss this distinction further in Appendix~\ref{app:Two-layer}.

\subsubsection{Experimental verification}
We empirically verify Theorem~\ref{thm:encoder-variational-balance} (and its
dynamical counterpart, Theorem~\ref{thm:encoder-balance-dynamics}) in a
minimal two-layer linear network with no architectural bottleneck,
$N_i=N_h=N_o=8$, so that $\mathbf{W}_1,\mathbf{W}_2\in\mathbb{R}^{8\times8}$ are both square
and (generically) invertible.

\emph{Task.} We use an orthogonal input design $\mathbf{X}=\sqrt{N_i}\,\mathbf{I}_{N_i}$ and a
fixed $\{0,\pm1/\sqrt{N_i}\}$-valued target matrix $\mathbf{Y}$ with hierarchical
block structure. The resulting least-squares mapping has
three distinct singular values (multiplicities $2,2,4$, reflecting the
hierarchy), letting us check the predicted balance across multiple
directions at once.

\emph{Training protocol.} In Fig.~\ref{appfig:linearreg} (A), both layers are initialized with unbalanced i.i.d.\
Gaussian weights, $\mathbf{W}_1(0),\mathbf{W}_2(0)\sim\mathcal N(0,\sigma^2)$, $\sigma=0.5$. The network is
trained by full-batch gradient descent on
$\mathcal{L}_{\mathrm{reg}}=\mathcal{L}_{\mathrm{task}}+\mathcal{R}_{\mathrm{LSAC}}$
($\eta=10^{-3}$, $T=3\times10^4$ steps), with the log-determinant term on the
\emph{encoder} Gram matrix, $-\gamma \log\det(\mathbf{W}_1\mathbf{W}_1^T)$, and weight decay
$\lambda_{reg}$ on both layers ($(\lambda_{reg},\gamma)=(0.1,0.2))$.

\emph{Verification protocol.} At $t=T$, we compare three block Gram matrices $
\mathbf{Q}=
\begin{pmatrix}
\wa^\top\wa & \boldsymbol{M}^\top\\
\boldsymbol{M} & \wb\wb^\top
\end{pmatrix}$,
 where $\boldsymbol{M}=\wb\wa,
$
 \textbf{(A)} $\mathbf{Q}_{\mathrm{emp}}$ from the regularized network;
\textbf{(B)} $\mathbf{Q}_{\mathrm{theory}}$, obtained by Theorem \ref{thm:singular}
Appendix~\ref{app:singular}, with
$\lambda_{\mathrm{bal}}=-\gamma/\lambda_{reg}=-2$; and \textbf{(C)}
$\mathbf{Q}_{\mathrm{bal}}$, from a network initialized exactly at this
$\lambda_{\mathrm{bal}}$-balanced solution and trained unregularized
($\lambda_{reg}=\gamma=0$), whose balance is conserved by construction
under unregularized gradient flow.

Fig.~\ref{appfig:linearreg} shows near-perfect agreement across all
three panels. The empirical imbalance eigenvalues concentrate tightly around
the predicted, negative fixed point,
$\lambda_i(\mathbf{\Delta}(T))\in[-2.005,-1.984]\approx\lambda_{\mathrm{bal}}$, and
the small relative error,
\[
e_{rel}
=
\frac{\lVert \mathbf{Q}_{\mathrm{emp}} - \mathbf{Q}_{\mathrm{theory}} \rVert_F}
{\lVert \mathbf{Q}_{\mathrm{theory}} \rVert_F}
=
0.0095,
\]
confirms that the encoder-side log-determinant regularizer drives the
network towards the predicted, collapse-preventing balanced fixed point,
without impeding task fit (final MSE $\approx5.47\times10^{-4}$).

\begin{figure}
\begin{center}
\includegraphics[width=0.9\textwidth]{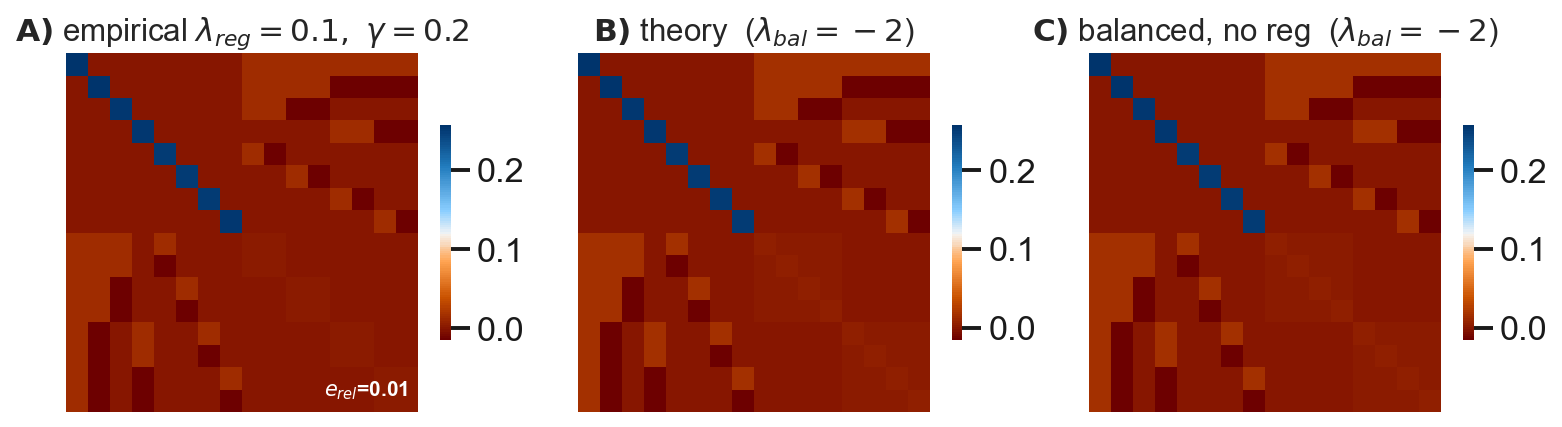}
\end{center}
\caption{
Block Gram matrix
$
\mathbf{Q}=
\begin{pmatrix}
\wa^\top\wa & \boldsymbol{M}^\top\\
\boldsymbol{M} & \wb\wb^\top
\end{pmatrix}$,
$\boldsymbol{M}=\wb\wa,
$
for a two-layer linear network.
\textbf{(A)} Training with symmetric $L_2$ weight decay
($\lambda_{\mathrm{reg}}=0.1$) and an encoder-side
negative log-determinant penalty ($\gamma=0.2$).
The regularizer drives the layer imbalance
$\mathbf{\Delta}=\wb^\top\wb-\wa\wa^\top$
towards
$\mathbf{\Delta}_\star=-(\gamma/\lambda_{\mathrm{reg}})
\mathbf{I}_{N_h}=-2\mathbf{I}_{N_h}$,
corresponding to $\lambda_{\mathrm{bal}}=-2$.
\textbf{(B)} Closed-form prediction from
Appendix~\ref{app:singular}.
\textbf{(C)} Unregularized control network initialized at
$\lambda_{\mathrm{bal}}=-2$.
This balance is conserved under gradient flow, providing a
comparison with the balance selected dynamically in~(A).
}
\label{appfig:linearreg}
\end{figure}

\subsection{From linear to nonlinear covariance regularization}

\subsubsection{From linear balance to nonlinear covariance regularization}
\label{sec:nonlinear-covariance-regularisation}

%The preceding linear analysis shows that constrained minimisers of
%the encoder-side regulariser satisfy
%\[
%\\wb^\top\wb-\wa\wa^\top
%\=
%\-\frac{\gamma}{\lambda_{\mathrm{reg}}}\mathbf{I}_{N_h},
%\\]
%\and therefore
%\\[
%\\wa\wa^\top
%\\succeq
%\\frac{\gamma}{\lambda_{\mathrm{reg}}}\mathbf{I}_{N_h}.
%\\]
%\Regularised gradient flow also converges towards this selected
%\balance. 

%These results motivate applying the same spectral
%regularisation directly to the covariance of nonlinear
%representations. We additionally penalise the representation mean,
%which need not vanish even when the inputs are centred.

For a linear encoder with centered, whitened inputs and $\varepsilon=0$, the covariance penalty reduces exactly to the encoder portion of the linear regularizer, with $\lambda_{\mathrm{cov}}=\lambda_{\mathrm{reg}}$. 

\begin{lemma}[Linear encoder mean and covariance]
\label{lem:linear-gram-covariance}
Let $\mathbf{x}\in\mathbb{R}^{N_i}$ have finite second moments, mean
$\mathbf{\boldsymbol{\mu}}_x=\mathbb{E}[\mathbf{x}]$, and covariance
$\boldsymbol{\Sigma}_x=\mathbb{E}[(\mathbf{x}-\boldsymbol{\mu}_x)(\mathbf{x}-\boldsymbol{\mu}_x)^\top]$.
For the bias-free linear representation $\mathbf{h}=\wa \mathbf{x}$, its mean
and covariance satisfy $\boldsymbol{\mu}=\wa\boldsymbol{\mu}_x, $
\text{}$
\mathbf{C}_h=\wa\boldsymbol{\Sigma}_x\wa^\top.$

In particular, for centered, whitened inputs,
$\boldsymbol{\mu}_x=\mathbf{0}$ and $\boldsymbol{\Sigma}_x=\mathbf{I}_{N_i}$, so $\boldsymbol{\mu}=\mathbf{0},
$ $
\mathbf{C}_h=\wa\wa^\top,$ and $
\operatorname{tr}(\mathbf{C}_h)=\|\wa\|_F^2.$
\end{lemma}

\begin{proof}
By linearity of expectation, $\boldsymbol{\mu}=\mathbb{E}[\wa \mathbf{x}]=\wa\boldsymbol{\mu}_x.$
Consequently, $\mathbf{h}-\boldsymbol{\mu}=\wa(\mathbf{x}-\boldsymbol{\mu}_x)$, giving $\mathbf{C}_h
=
\mathbb{E}[(\mathbf{h}-\boldsymbol{\mu})(\mathbf{h}-\boldsymbol{\mu})^\top]
=
\wa\boldsymbol{\Sigma}_x\wa^\top.$
For centered, whitened inputs, these identities reduce to
$\boldsymbol{\mu}=0$ and $\mathbf{C}_h=\wa\wa^\top$. Taking the trace yields
$\operatorname{tr}(\mathbf{C}_h)=\|\wa\|_F^2$.
\end{proof}

\paragraph{Recovery of the linear encoder regularizer.}
Recall that the nonlinear regularizer is
\[
\mathcal{R}_{\mathrm{SAC}}(\boldsymbol{\mu},\mathbf{C}_h)
=
\frac{\lambda_{\mathrm{mean}}}{2}\|\boldsymbol{\mu}\|_2^2
+
\frac{\lambda_{\mathrm{cov}}}{2}\operatorname{tr}(\mathbf{C}_h)
-
\frac{\gamma}{2}
\log\det(\mathbf{C}_h+\varepsilon\mathbf{I}_{N_h}).
\]
Assume $\wa\wa^\top\succ0$, centered and whitened inputs,
$\varepsilon=0$, and
$\lambda_{\mathrm{cov}}=\lambda_{\mathrm{reg}}$.
Lemma~\ref{lem:linear-gram-covariance} then gives
\[
\mathcal{R}_{\mathrm{SAC}}(\boldsymbol{\mu},\mathbf{C}_h)
=
\frac{\lambda_{\mathrm{reg}}}{2}\|\wa\|_F^2
-
\frac{\gamma}{2}\log\det(\wa\wa^\top).
\]
The mean penalty vanishes because $\boldsymbol{\mu}=\mathbf{0}$, regardless of the
positive value of $\lambda_{\mathrm{mean}}$.
Thus, the nonlinear regularizer reduces exactly to the encoder
portion of the linear regularizer:
\[
\mathcal{R}_{\mathrm{LSAC}}(\wa,\wb)
=
\mathcal{R}_{\mathrm{SAC}}(\boldsymbol{\mu},\mathbf{C}_h)
+
\frac{\lambda_{\mathrm{reg}}}{2}\|\wb\|_F^2.
\]

\paragraph{Role of the mean penalty.}
For nonlinear encoders, centered inputs need not produce centered
representations. Moreover, covariance is invariant under a common
translation of all representations, so neither the trace nor the
log-determinant term controls their mean. The additional
$\|\boldsymbol{\mu}\|_2^2$ penalty controls this translation by favoring
zero-mean representations. It therefore extends the regularizer
without changing its reduction to the bias-free linear setting
with centered, whitened inputs. The same identities hold for empirical means and covariances.

\subsubsection{Log-determinant covariance regularizer with the mean }
\label{app:logdet-covariance-regulariser}

Let $\mathbf{h}=f_\theta(\mathbf{x})\in\mathbb{R}^{N_h}$ denote the representation produced by a possibly nonlinear
encoder. For a minibatch of size \(B\), collect the representations row-wise in $\mathbf{H}\in\mathbb{R}^{B\times N_h}.$
Define the empirical mean, centering matrix, centered representations, and empirical covariance as 
$
\boldsymbol{\mu}=\frac{1}{B}\mathbf{H}^\top\mathbf{1},
$ $
\mathbf{P}_B
=
\mathbf{I}_B-\frac{1}{B}\mathbf{1}\mathbf{1}^\top,
$ $
\widetilde{\mathbf{H}}=\mathbf{P}_B\mathbf{H},
$ $
\mathbf{C}_h
=
\frac{1}{B}\widetilde{\mathbf{H}}^\top\widetilde{\mathbf{H}}
\succeq0.$ The linear analysis never needs to control the
representation's mean: for a bias-free linear encoder $\mathbf{h}=\wa \mathbf{x}$ with centered input, $\mathbb
E[\mathbf{h}]=\wa\,\mathbb E[\mathbf{x}]=\mathbf{0}$ identically, so the mean carries no independent degree of freedom for
$\mathcal R_{\mathrm{SAC}}$ to regularize. This guarantee does not survive the transfer to a
general nonlinear (or biased) encoder: a bias term, or the nonlinearity itself (e.g.\ a ReLU
clipping one side of an otherwise symmetric pre-activation), can give $\mathbf{h}$ a nonzero mean that
moves independently of $\mathbf{C}_h$ -- invisible to both $\operatorname{tr}(\mathbf{C}_h)$ and $\log\det(\mathbf{C}_h)$,
since $\mathbf{C}_h$ is computed from the \emph{centered} representation by construction. We therefore
regularize $\boldsymbol{\mu}$ directly alongside $\mathbf{C}_h$. Motivated by the linear encoder regularizer, we give the following definition.

\begin{definition}[Nonlinear SACReg]
Let $\lambda_{\mathrm{mean}}>0$,
$\lambda_{\mathrm{cov}}>0$, $\gamma>0$, and $\varepsilon\geq0$.
Assume that $\mathbf{C}_h+\varepsilon\mathbf{I}_{N_h}\succ0$.
We define
\begin{equation}
\mathcal{R}_{\mathrm{SAC}}(\boldsymbol{\mu},\mathbf{C}_h)
=
\frac{\lambda_{\mathrm{mean}}}{2}\|\boldsymbol{\mu}\|_2^2
+
\frac{\lambda_{\mathrm{cov}}}{2}\operatorname{tr}(\mathbf{C}_h)
-
\frac{\gamma}{2}
\log\det\left(\mathbf{C}_h+\varepsilon\mathbf{I}_{N_h}\right).
\label{eq:nonlinear-covariance-regulariser}
\end{equation}
\end{definition}
 The mean term controls the representation's location, the trace term controls its overall scale, and
the negative log-determinant penalizes small covariance eigenvalues.
The regularizer constrains only the representation mean and covariance; it does not directly constrain higher-order moments or enforce Gaussianity.

To understand the inductive bias of the regularizer, we consider its minimizer in isolation. This does not imply that the mean and covariance of the complete learning objective must attain this value, since the task loss may introduce competing constraints. First, we write the regularizer in spectral form.

\begin{proposition}[Spectral form]
\label{prop:nonlinear-covariance-spectral-form}
Let $\lambda_1(\mathbf{C}_h),\ldots,\lambda_{N_h}(\mathbf{C}_h)$ be the eigenvalues of \(\mathbf{C}_h\). Then
\begin{equation}
\mathcal{R}_{\mathrm{SAC}}(\boldsymbol{\mu},\mathbf{C}_h)
=
\frac{\lambda_{\mathrm{mean}}}{2}\|\boldsymbol{\mu}\|^2
+
\frac{1}{2}
\sum_{i=1}^{N_h}
\left[
    \lambda_{\mathrm{cov}}\lambda_i(\mathbf{C}_h)
    -
    \gamma\log\left(
        \lambda_i(\mathbf{C}_h)+\varepsilon
    \right)
\right].
\label{eq:nonlinear-covariance-eigenvalue-form}
\end{equation}
\end{proposition}
\begin{proof}
Because \(\mathbf{C}_h\) is symmetric and positive semidefinite, it admits an eigendecomposition
$\mathbf{C}_h=\mathbf{Q}\operatorname{diag}(\lambda_1,\ldots,\lambda_{N_h})\mathbf{Q}^\top$. Consequently,
$\operatorname{tr}(\mathbf{C}_h)=\sum_{i=1}^{N_h}\lambda_i$ and $\log\det(\mathbf{C}_h+\varepsilon \mathbf{I}_{N_h})
=\sum_{i=1}^{N_h}\log(\lambda_i+\varepsilon)$. The mean term does not depend on $\mathbf{C}_h$. Substitution
into \eqref{eq:nonlinear-covariance-regulariser} proves the result.
\end{proof}

We can now write the optimal mean and covariance.

\begin{theorem}[Optimal mean and covariance]
\label{thm:nonlinear-optimal-covariance}
The unique minimizer of \(\mathcal{R}_{\mathrm{SAC}}\) over $\boldsymbol{\mu}\in\mathbb R^{N_h},\ \mathbf{C}_h\succeq0$  such that $\mathbf{C}_h+\varepsilon\mathbf{I}_{N_h}\succ 0$ is
\begin{equation}
\boxed{
\boldsymbol{\mu}^\star=\mathbf{0},
\qquad
\mathbf{C}_h^\star
=
\left(
    \frac{\gamma}{\lambda_{\mathrm{cov}}}
    -
    \varepsilon
\right)_{+}\mathbf{I}_{N_h}
},
\label{eq:nonlinear-optimal-covariance}
\end{equation}
where $(a)_{+}=\max\{a,0\}$. Therefore, the covariance-level optimum is full-rank if and only if
$\gamma>\lambda_{\mathrm{cov}}\varepsilon$, independently of $\boldsymbol{\mu}^\star$.
\end{theorem}

\begin{proof}
By Proposition~\ref{prop:nonlinear-covariance-spectral-form}, $\mathcal{R}_{\mathrm{SAC}}$ is a sum
of a term depending only on $\boldsymbol{\mu}$ and a term depending only on $\mathbf{C}_h$, so the two may be minimized
separately. The mean term $\frac{\lambda_{\mathrm{mean}}}{2}\|\boldsymbol{\mu}\|^2$ is uniquely minimized at
$\boldsymbol{\mu}^\star=\mathbf{0}$. For the covariance term, each eigenvalue independently minimizes the strictly
convex function
\[
\phi(s)
=
\frac{\lambda_{\mathrm{cov}}}{2}s
-
\frac{\gamma}{2}\log(s+\varepsilon),
\qquad
s\geq0 \quad , s+\varepsilon \geq 0
\]
Its derivatives are $\phi'(s)=\frac{\lambda_{\mathrm{cov}}}{2}-\frac{\gamma}{2(s+\varepsilon)}$ and
$\phi''(s)=\frac{\gamma}{2(s+\varepsilon)^2}>0$. The unconstrained stationary point satisfies
$\phi'(s)=0 \Leftrightarrow s=\frac{\gamma}{\lambda_{\mathrm{cov}}}-\varepsilon$. If $\gamma/\lambda_{\mathrm{cov}}-\varepsilon>0$, this stationary point is the unique minimizer. Otherwise, $\varepsilon>0$, so $s=0$ is admissible and is the unique minimizer.
\end{proof}

\begin{corollary}[Isotropic non-collapsed target, centered]
\label{cor:nonlinear-isotropic-target}
In the idealized case \(\varepsilon=0\),
\begin{equation}
\boxed{
\boldsymbol{\mu}^\star=\mathbf{0},
\qquad
\mathbf{C}_h^\star
=
\frac{\gamma}{\lambda_{\mathrm{cov}}}\mathbf{I}_{N_h}
}.
\label{eq:nonlinear-isotropic-covariance}
\end{equation}
Thus, the regularizer simultaneously:
\begin{enumerate}
    \item centers the representation at the origin;
    \item prevents the covariance eigenvalues from vanishing;
    \item equalizes the covariance spectrum;
    \item controls the overall representation scale; and
    \item selects a condition number equal to one.
\end{enumerate}

\end{corollary}

Next we derive the non-collapse property of the regularizer.
 
\begin{theorem}[Mean boundedness and strict covariance non-collapse]
\label{thm:strict-covariance-non-collapse}
Suppose that \(\varepsilon=0\). Then
\begin{equation}
\lambda_{\min}(\mathbf{C}_h)\to0
\quad\Longrightarrow\quad
\mathcal{R}_{\mathrm{SAC}}(\boldsymbol{\mu},\mathbf{C}_h)\to+\infty,
\label{eq:covariance-collapse-barrier}
\end{equation}
\begin{equation}
\lambda_{\max}(\mathbf{C}_h)\to+\infty
\quad\Longrightarrow\quad
\mathcal{R}_{\mathrm{SAC}}(\boldsymbol{\mu},\mathbf{C}_h)\to+\infty,
\label{eq:covariance-scale-barrier_2}
\end{equation}
\begin{equation}
\|\boldsymbol{\mu}\|\to\infty
\quad\Longrightarrow\quad
\mathcal{R}_{\mathrm{SAC}}(\boldsymbol{\mu},\mathbf{C}_h)\to+\infty.
\label{eq:mean-blowup-barrier}
\end{equation}

Consequently, any fixed finite upper bound on the regularizer
bounds the mean and keeps all covariance eigenvalues bounded
above and away from zero. Specifically, for every finite $K$,
there exist constants $0<m_K\le M_K<\infty$ and
$0\le M'_K<\infty$ such that
\begin{equation}
\label{eq:covariance-scale-barrier}
R_{\mathrm{SAC}}(\boldsymbol{\mu},\mathbf{C}_h)\le K
\quad\Longrightarrow\quad
m_K \mathbf{I}_{N_h}\preceq \mathbf{C}_h\preceq M_K \mathbf{I}_{N_h},
\qquad
\|\boldsymbol{\mu}\|_2\le M'_K.
\end{equation}
\end{theorem}

\begin{proof}
Suppose that \(\varepsilon=0\) and \(\mathbf{C}_h\succ0\). Let $\lambda_1(\mathbf{C}_h),\ldots,\lambda_{N_h}(\mathbf{C}_h)>0$
denote the eigenvalues of \(\mathbf{C}_h\). Since $\operatorname{tr}(\mathbf{C}_h)=\sum_{i=1}^{N_h}\lambda_i(\mathbf{C}_h)$ and
$\log\det(\mathbf{C}_h)=\sum_{i=1}^{N_h}\log\lambda_i(\mathbf{C}_h)$, write
\[
\mathcal{R}_{\mathrm{SAC}}(\boldsymbol{\mu},\mathbf{C}_h) = \frac{\lambda_{\mathrm{mean}}}{2}\|\boldsymbol{\mu}\|^2+\sum_{i=1}^{N_h}
\phi\left(\lambda_i(\mathbf{C}_h)\right),
\qquad
\phi(s)
=
\frac{\lambda_{\mathrm{cov}}}{2}s
-
\frac{\gamma}{2}\log s,
\quad
s>0.
\]
We first show that \(\phi\) is bounded below. Its derivatives are
$\phi'(s)=\frac{\lambda_{\mathrm{cov}}}{2}-\frac{\gamma}{2s}$ and
$\phi''(s)=\frac{\gamma}{2s^2}>0$, so \(\phi\) is strictly convex, with unique stationary point
$s=\gamma/\lambda_{\mathrm{cov}}$ and global minimum
\[
b
:=
\min_{s>0}\phi(s)
=
\phi\!\left(\frac{\gamma}{\lambda_{\mathrm{cov}}}\right)
=
\frac{\gamma}{2}
-
\frac{\gamma}{2}
\log\!\left(\frac{\gamma}{\lambda_{\mathrm{cov}}}\right).
\]
In particular, $\phi(s)\geq b$ for every $s>0$, so $\sum_{i=1}^{N_h}\phi(\lambda_i(\mathbf{C}_h))\geq N_h b$ for
every $\mathbf{C}_h\succ0$.

\emph{Covariance collapse.} Without loss of generality order the eigenvalues so that
$\lambda_1(\mathbf{C}_h)=\lambda_{\min}(\mathbf{C}_h)$. Then
\[
\mathcal{R}_{\mathrm{SAC}}(\boldsymbol{\mu},\mathbf{C}_h)
\;\geq\;
\phi\left(\lambda_{\min}(\mathbf{C}_h)\right)
+
(N_h-1)b.
\]
As $s\downarrow0$, $\frac{\lambda_{\mathrm{cov}}}{2}s\to0$ and $-\frac{\gamma}{2}\log s\to+\infty$,
so $\lim_{s\downarrow0}\phi(s)=+\infty$. Since $(N_h-1)b$ is a fixed finite constant,
$\lambda_{\min}(\mathbf{C}_h)\to0\Rightarrow\mathcal{R}_{\mathrm{SAC}}(\boldsymbol{\mu},\mathbf{C}_h)\to+\infty$, proving
\eqref{eq:covariance-collapse-barrier}.

\emph{Covariance blowup.} As $s\to+\infty$, $\frac{\lambda_{\mathrm{cov}}}{2}s\to+\infty$ dominates
$-\frac{\gamma}{2}\log s$, so $\phi(s)\to+\infty$; the same argument with $\lambda_{\max}(\mathbf{C}_h)$ in
place of $\lambda_{\min}(\mathbf{C}_h)$ proves \eqref{eq:covariance-scale-barrier_2}.

\emph{Mean blowup.}
Since $\phi(s)\geq b$ for every $s>0$, we have
\[
\mathcal R_{\mathrm{SAC}}(\boldsymbol{\mu},\mathbf{C}_h)
\geq
\frac{\lambda_{\mathrm{mean}}}{2}\|\boldsymbol{\mu}\|_2^2+N_hb
\qquad\text{for every }\mathbf{C}_h\succ0.
\]
Because $\lambda_{\mathrm{mean}}>0$ and $b$ is finite,
$\|\boldsymbol{\mu}\|_2\to\infty$ implies
$\mathcal R_{\mathrm{SAC}}(\boldsymbol{\mu},\mathbf{C}_h)\to+\infty$,
even when $\mathbf{C}_h$ varies.
This proves \eqref{eq:mean-blowup-barrier}.

The preceding limits imply that, for any fixed finite $K$,
the condition $\mathcal{R}_{\mathrm{SAC}}(\boldsymbol{\mu},\mathbf{C}_h)\leq K$
keeps the covariance eigenvalues bounded above and away from
zero, and bounds $\|\boldsymbol{\mu}\|$. Otherwise, a sequence satisfying
this condition would approach covariance collapse, unbounded
covariance, or unbounded mean, forcing the regularizer to
diverge and contradicting its upper bound $K$.
This proves \eqref{eq:covariance-scale-barrier}.
\end{proof}

\begin{remark}[Effect of the numerical stabilizer]
The regularizer prevents covariance collapse when \(\varepsilon=0\), but only discourages it when \(\varepsilon>0\). In the latter case, its unique covariance-level minimizer is still full-rank provided \(\gamma>\lambda_{\mathrm{cov}}\varepsilon\) (Theorem~\ref{thm:nonlinear-optimal-covariance}). Mean boundedness remains guaranteed in both cases.
\end{remark}

\begin{remark}[Scope of the nonlinear extension]
The nonlinear covariance regularizer is motivated by the exact linear balance analysis, but it
does not imply that the linear balance identity $\wb^\top\wb-\wa\wa^\top
=-\frac{\gamma}{\lambda_{\mathrm{reg}}}\mathbf{I}_{N_h}$ continues to hold for a general nonlinear encoder. The
nonlinear parameter dynamics are modified by the encoder Jacobian, and there is generally no
conserved difference between the decoder Gram matrix and the representation covariance. 
\end{remark}

% =========================
% Rich feature learning
% =========================
\begin{figure*}[t]
    \centering
\includegraphics[width=.8\textwidth]{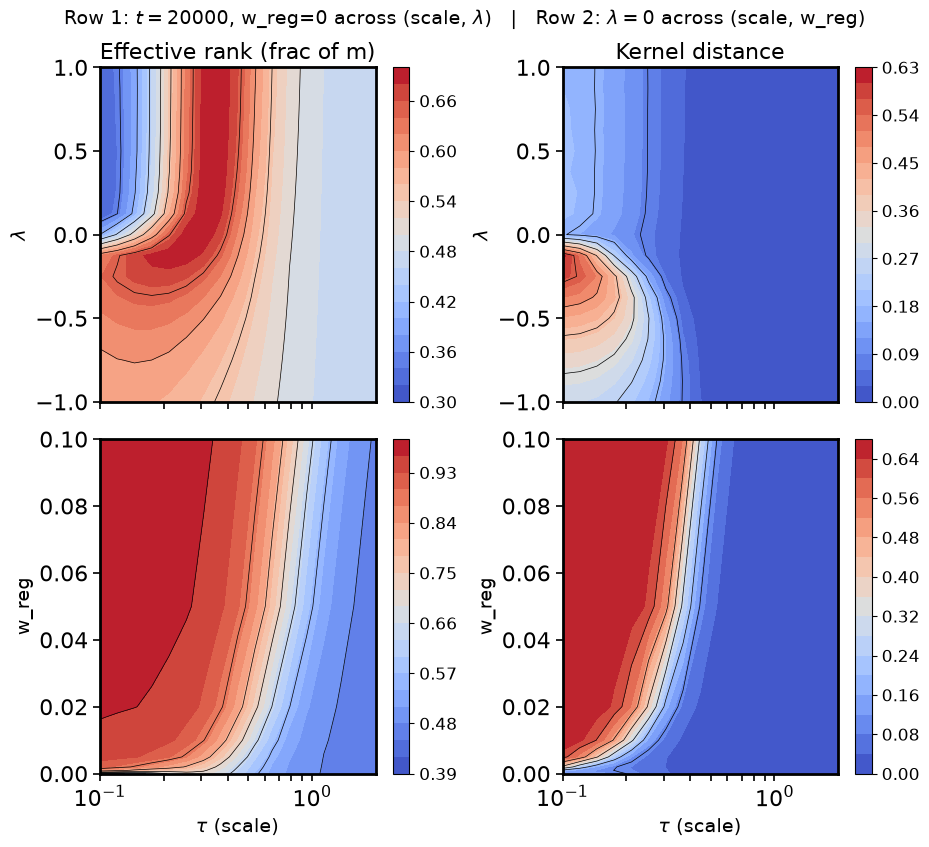}
    \caption{Effective rank and kernel distance of the NTK from initialization. \textbf{Top row:} Both metrics as functions of the scale $\tau$ and relative scale $\lambda$ at step $t = 20{,}000$, without regularization ($w_{\mathrm{reg}} = 0$). \textbf{Bottom row:} The same metrics at the final checkpoint as functions of the scale $\tau$ and regularization weight $w_{\mathrm{reg}}$, with $\lambda = 0$ fixed.    }
    \label{fig:rich-Relu}
\end{figure*}

\subsubsection{Feature learning: Two-layer ReLU experiment}
\label{app:Two-layer}

\paragraph{Task and data.}
We study nonlinear teacher--student regression with a fixed,
untrained two-layer ReLU teacher of hidden width \(N_{h_0}=3\).
The teacher is
\[
f_\star(\mathbf{x})
=
\sum_{j=1}^{N_{h_0}} a^\star_j
\sigma\!\left((\mathbf{w}^\star_j)^\top \mathbf{x}\right),
\qquad
\sigma(z)=\max\{0,z\},
\]
where \(\mathbf{w}^\star_j\) are sampled independently and uniformly
from the unit sphere \(\mathbb{S}^{d-1}\), and \(a^\star_j\) are
independent Rademacher random variables.
For each run, we sample \(n=1{,}000\) training inputs independently
from \(\mathbb{S}^{N_i-1}\) and assign noiseless labels
\(y_i=f_\star(\mathbf{x}_i)\).
A separate test set of \(n_{\mathrm{test}}=10{,}000\) inputs,
generated with seed \(201\), is reused at every checkpoint for
both test-loss and representation-rank measurements.

\paragraph{Student architecture.}
The student is a two-layer, bias-free ReLU network with
hidden width \(N_h=50\):
\[
\mathbf{h}_\theta(\mathbf{x})
=
\sigma(\mathbf{W}\mathbf{x})
\in\mathbb{R}^{N_h},
\qquad
f_\theta(\mathbf{x})
=
\mathbf{a}^\top\mathbf{h}_\theta(\mathbf{x}).
\]
We measure the rank of the post-ReLU representation
\(\mathbf{h}_\theta(\mathbf{x})\), the network's only hidden layer.
The student uses a symmetric initialization: hidden units
form \(N_h/2\) pairs with identical input weights and
opposite readout weights,
\[
\mathbf{w}_{j+N_h/2}=\mathbf{w}_j,
\qquad
a_{j+N_h/2}=-a_j,
\qquad
j=1,\ldots,N_h/2.
\]
This ensures \(f_{\theta_0}(\mathbf{x})=\mathbf{0}\) for every input
while retaining a nontrivial initial neural tangent kernel (NTK).

\paragraph{Initialization.}
Two parameters control initialization: an overall scale \(\tau>0\)
and a per-neuron imbalance \(\lambda\).
For each independently sampled pair, we initialize
\[
\mathbf{w}_j=\frac{\tau}{\alpha}\mathbf{u}_j,
\qquad
a_j=\tau\alpha s_j,
\]
where \(\mathbf{u}_j\) is uniform on \(\mathbb{S}^{N_i-1}\),
\(s_j\in\{-1,+1\}\) is a Rademacher random variable, and
\(\alpha>0\) satisfies
\[
\lambda
=
a_j^2-\|\mathbf{w}_j\|_2^2
=
\tau^2\left(\alpha^2-\alpha^{-2}\right).
\]
Equivalently,
\[
\alpha
=
\left(
\frac{
\lambda/\tau^2+
\sqrt{(\lambda/\tau^2)^2+4}
}{2}
\right)^{1/2}.
\]
Thus, negative \(\lambda\) corresponds to larger
encoder weights, whereas positive \(\lambda\)
corresponds to larger readout weights.

\paragraph{Objective and optimization.}
We train the student using full-batch gradient descent on
the mean squared error,
\[
\mathcal{L}_{\mathrm{MSE}}(\theta)
=
\frac{1}{n}\sum_{i=1}^{n}
\left(f_\theta(\mathbf{x}_i)-y_i\right)^2,
\]
with an additional regularization term in the regularized
experiments.
The learning rate is scaled as
\(\eta=5\times10^{-3}/\tau^2\).
We apply the log-determinant covariance regularizer directly
to the hidden representation of the toy two-layer ReLU
student network,
\[
\mathbf{h}
=
\mathrm{ReLU}(\mathbf{W}_1\mathbf{x})
\in\mathbb{R}^{N_h}
\]
(here \(N_h=50\)), immediately after the nonlinearity and
before the readout layer.
We use the implementation of the regularizer described in
Appendix~\ref{app:slicing-theory}.

We use \(\lambda_{\mathrm{cov}}=\gamma\).
Therefore, the regularized training objective is
\[
\mathcal{L}
=
\mathcal{L}_{\mathrm{task}}
+
w_{\mathrm{reg}}\cdot
\mathcal{R}_{\mathrm{SAC}}(\boldsymbol{\mu},\mathbf{C}_h),
\]
with \(\mathcal{L}_{\mathrm{task}}\) the mean-squared error
between the student's output and the teacher's labels, and
\(w_{\mathrm{reg}}\in\{0,0.001,0.005,0.02,0.1\}\)
swept as the regularizer strength.

\textbf{Feature learning.}
Using the two-layer ReLU network setup of \cite{kunin2024get},
we investigate how the rank effects and the learning regime
of the initial imbalance \(\lambda\) relate to those of the
SAC regularizer. Linear theory suggests that a negative
relative scale should favor the same anti-collapse behavior
promoted by SAC regularization.

\begin{itemize}

\item \textit{Rank.}
There is no exact correspondence between the initial imbalance
\(\lambda\) and the regularization weight \(w_{\mathrm{reg}}\).
The regularizer appears to exert a substantially stronger
effect on rank (see Fig.~\ref{fig:rich-Relu}).
Nevertheless, both negative \(\lambda\)
(relative to \(\lambda\geq0\)) and positive \(w_{\mathrm{reg}}\)
favor higher rank in the Rich regime, supporting a qualitative
connection between initial imbalance and anti-collapse
regularization (see Fig.~\ref{fig:rich-Relu}).
Interestingly, the rank effect of negative \(\lambda\) emerges
during training rather than reflecting a decaying initial
offset. All negative-\(\lambda\) runs with \(w_{\mathrm{reg}}=0\)
begin at approximately the same effective rank fraction,
\(\mathtt{effrank\_frac}=0.467\), consistent with their shared
initialization (see Table~\ref{tab:effective-rank}).
This quantity then increases monotonically throughout training,
indicating that the higher final rank develops through the
training dynamics.

\item \textit{Learning regime.}
We characterize the learning regime by measuring the departure
of the neural tangent kernel (NTK) from its value at
initialization, quantified by kernel distance.
Increasing \(w_{\mathrm{reg}}\) leads to an increase in kernel
distance across the initialization scales considered, indicating
that stronger regularization promotes greater kernel evolution
(see Fig.~\ref{fig:rich-Relu}).
It also extends the range of initialization scales over which
feature learning occurs, although kernel evolution remains
limited at sufficiently large scales
(see Fig.~\ref{fig:rich-Relu}).
At small initialization scales, varying \(\lambda\) alone yields
kernel distance values between \(0.16\) and \(0.58\), comparable
to the range obtained by varying \(w_{\mathrm{reg}}\).
At large initialization scales, however, both parameters have
little effect on kernel evolution.
These results suggest that \(\lambda\) primarily modulates
the degree of feature learning within an already rich regime,
whereas \(w_{\mathrm{reg}}\) can also promote departures from
otherwise lazy dynamics.
Altogether, we find that regularization promotes feature
learning across a range of initialization scales, with stronger
regularization associated with more pronounced feature learning
(see Fig.~\ref{fig:rich-Relu} in Appendix~\ref{app:Two-layer}).
A similar trend emerges in the self-supervised setting.
We show that both the backbone representations and the empirical
NTK move substantially from initialization during training,
consistent with rich learning dynamics
(Fig.~\ref{fig:rich-learning} and Appendix~\ref{app:rich-vit}).

\end{itemize}

\begin{table}[t]
\centering
\caption{Effective rank fraction across training steps for different
values of layer imbalance $\lambda$.}
\label{tab:effective-rank}
\begin{tabular}{r|rrrrrr}
\hline
 & \multicolumn{6}{c}{Training step} \\
$\lambda$ & $0$ & $100$ & $1000$ & $5000$ & $10000$ & $20000$ \\
\hline
$-1.00$ & 0.467 & 0.467 & 0.484 & 0.554 & 0.574 & 0.581 \\
$-0.75$ & 0.467 & 0.467 & 0.487 & 0.563 & 0.587 & 0.595 \\
$-0.50$ & 0.467 & 0.467 & 0.492 & 0.577 & 0.606 & 0.616 \\
$-0.25$ & 0.467 & 0.469 & 0.500 & 0.596 & 0.627 & 0.639 \\
\hline
\end{tabular}
\end{table}

% =========================
% Rich feature learning
% =========================
\begin{figure*}[t]
    \centering
    \includegraphics[width=\textwidth]{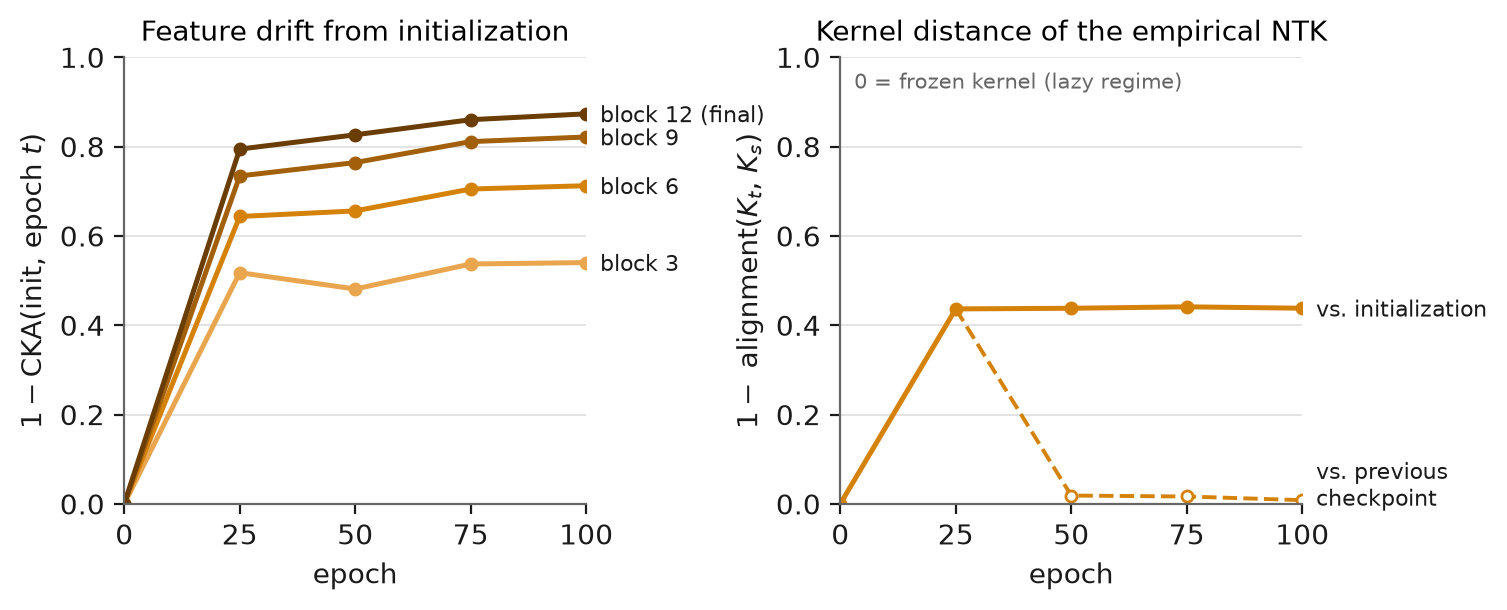}
    % \caption{
    % Training remains in the rich feature-learning regime. Backbone features move substantially from initialization, and the empirical tangent kernel differs strongly from its initial value.
    % }
    \caption{
    \textbf{Feature and kernel drift during self-supervised training.} ImageNet-100, ViT-S/16 trained for 100 epochs with SACReg applied at the backbone. \textbf{Left:} feature drift from initialization, measured by $1-\mathrm{CKA}$ at the CLS token after transformer blocks 3, 6, 9, and 12. \textbf{Right:} empirical NTK distance from initialization (solid) and from the previous checkpoint (dashed). Representations drift increasingly with depth, while the NTK moves substantially early in training and subsequently stabilizes, consistent with rich rather than lazy learning dynamics.
    }
    \label{fig:rich-learning}
\end{figure*}

\subsubsection{Feature learning in the ViT}
\label{app:rich-vit}
We additionally measure feature and kernel drift during self-supervised training
on ImageNet-100. We train a ViT-S/16 for 100 epochs with SACReg applied to the
backbone representation and evaluate all checkpoints on fixed held-out images.

\paragraph{Feature drift.}
For 4,096 validation images, let $\mathbf{X}_0,\mathbf{X}_t\in\mathbb{R}^{4096\times384}$
denote the CLS representations at initialization and epoch $t$, respectively,
at a given transformer block. We measure their dissimilarity using one minus
linear CKA \citep{kornblith2019similarity},
\[
1-\mathrm{CKA}(\mathbf{X}_0,\mathbf{X}_t), \qquad
\mathrm{CKA}(\mathbf{X},\mathbf{Y})
=
\frac{\|\mathbf{X}_c^\top \mathbf{Y}_c\|_F^2}
{\|\mathbf{X}_c^\top \mathbf{X}_c\|_F\,\|\mathbf{Y}_c^\top \mathbf{Y}_c\|_F},
\]
where $\mathbf{X}_c$ and $\mathbf{Y}_c$ are column-centered. We report this quantity after
blocks 3, 6, 9, and 12. Linear CKA is invariant to orthogonal transformations
and isotropic rescaling, so the measure captures changes in representation
geometry rather than changes in feature coordinates or overall scale.

\paragraph{Kernel drift.}
We estimate the empirical NTK of the backbone output with respect to the model parameters on a fixed subset of 256 validation images. To avoid constructing the
full vector-valued NTK, we use eight fixed unit-norm random directions
$\mathbf{v}_k\in\mathbb{R}^{384}$ in output space and compute
\[
\mathbf{g}_{ik}^{(t)}
=
\nabla_\theta\!\left(\mathbf{v}_k^\top \mathbf{h}_\theta(\mathbf{x}_i)\right),
\qquad
\mathbf{K}_t[i,j]
=
\frac{1}{8}\sum_{k=1}^{8}
\left\langle \mathbf{g}_{ik}^{(t)},\mathbf{g}_{jk}^{(t)}\right\rangle .
\]
The same images and random directions are used at every checkpoint. Up to a
constant factor, $\mathbf{K}_t$ is an unbiased estimate of the trace NTK; this factor
cancels in the normalized comparison. We measure kernel drift as
\[
d(\mathbf{K}_t,\mathbf{K}_s)
=
1-
\frac{\langle \mathbf{K}_t,\mathbf{K}_s\rangle_F}
{\|\mathbf{K}_t\|_F\|\mathbf{K}_s\|_F},
\]
and report distance both from initialization ($s=0$) and from the preceding checkpoint. In the lazy-training limit, the NTK remains fixed, so its distance from initialization stays at zero.

Figure~\ref{fig:rich-learning} shows substantial drift in both representations and the empirical NTK. Feature drift increases with depth, reaching its largest value at the final transformer block. The NTK moves away from its initialization during
the first 25 epochs and then changes only slightly between subsequent checkpoints. Together, these measurements are consistent with rich feature learning.

\section{Implementation of the regularizer}
\label{app:slicing}

\begin{figure*}[t]
    \centering
\includegraphics[width=.8\textwidth]{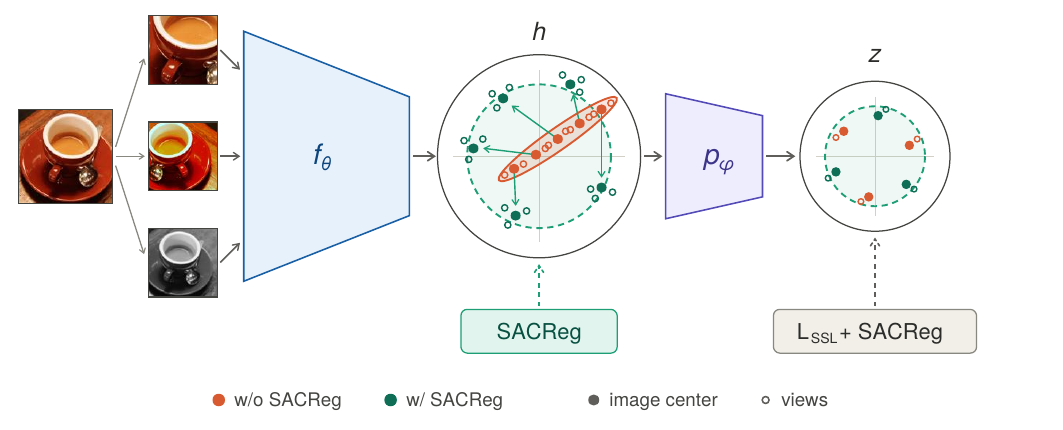}
    \caption{\textbf{Schematic overview of $\lambda$-JEPA.} Augmented views are mapped by the encoder ($f_\theta$) to hidden representations (h), then by ($p_\varphi$) to embeddings (z). SACReg applies regularization directly in the hidden space to prevent per-view collapse, encouraging a more dispersed representation distribution (green) compared with training without SACReg (orange). Filled circles denote image centers; open circles denote individual views. }
    \label{fig:Schematic_SACReg}
\end{figure*}

\paragraph{Invariance objective.}
For image $i$ with $V$ augmented views, we use the mean squared distance between all pairs of projected representations,
\[
\mathcal L_{\mathrm{SSL}}
=
\frac{1}{B}\sum_{i=1}^{B}
\frac{2}{V(V-1)}
\sum_{j<k}
\frac{1}{d_z}
\left\|\mathbf{z}_i^{(j)}-\mathbf{z}_i^{(k)}\right\|_2^2 ,
\]
where $B$ is the number of images in the batch and $d_z$ is the projector dimension.
This objective contains neither negatives nor an anti-collapse term.
Using $\bar{\mathbf{z}}_i=\frac{1}{V}\sum_{j=1}^{V}\mathbf{z}_i^{(j)}$, it is equivalently
$\frac{2V}{V-1}$ times the average squared distance from each view to its image center.
All runs use global views of a single size: $V=4$ on ImageNet-100 and $V=6$ on ImageNet-1k and video.
The invariance and regularization weights are set jointly as described in Appendix~\ref{app:exp-details}.

\paragraph{Slicing.}
For $B$ view centers $\mathbf{R}=\{\mathbf{r}_i\}_{i=1}^{B}\subset\mathbb R^d$,
the centered empirical covariance has rank at most $B-1$.
Thus, when $d\geq B$, the full log-determinant is degenerate.
At each step, we draw a random orthonormal projection
$\mathbf{U}\in\mathbb R^{d\times d'}$ using a thin QR factorization of a Gaussian matrix and evaluate SACReg on the projected centers
$\{\mathbf{U}^\top \mathbf{r}_i\}_{i=1}^{B}$.
The same projection is shared across GPUs and resampled every step. Within the slice, we compute the projected mean and covariance and use $\varepsilon=10^{-4}$.
We divide the regularizer by $d'$, making its scale approximately independent of the slice width.

We use $d'=128$ after the projector ($d_z=256$).
At the backbone, we use one third of the feature width:
$d'=128$ for ViT-S ($d=384$) and $d'=256$ for ViT-B ($d=768$).
The backbone and projector regularizers use independent random slices.

\subsection{Relation to the full-covariance regularizer}
\label{app:slicing-theory}

For clarity, consider the idealized case $\varepsilon=0$.
Let $\boldsymbol{\mu}\in\mathbb R^d$ and $\boldsymbol{\Sigma}\succ0$ denote the mean and covariance
of the view centers. The full regularizer per dimension is
\[
\mathcal R_{\mathrm{full}}(\boldsymbol{\mu},\boldsymbol{\Sigma})
=
\frac{1}{2d}
\left(
\|\boldsymbol{\mu}\|_2^2+\operatorname{tr}\boldsymbol{\Sigma}-d-\log\det\boldsymbol{\Sigma}
\right).
\]

For an orthonormal frame $\mathbf{U}\in\mathbb R^{d\times d'}$, the projected
centers have mean $\mathbf{U}^\top\boldsymbol{\mu}$ and covariance $\mathbf{U}^\top\boldsymbol{\Sigma}\mathbf{U}$, giving
the sliced regularizer
\[
\mathcal R_{\mathbf{U}}(\boldsymbol{\mu},\boldsymbol{\Sigma})
=
\frac{1}{2d'}
\left(
\|\mathbf{U}^\top\boldsymbol{\mu}\|_2^2
+\operatorname{tr}(\mathbf{U}^\top\boldsymbol{\Sigma}\mathbf{U})
-d'
-\log\det(\mathbf{U}^\top\boldsymbol{\Sigma}\mathbf{U})
\right).
\]
Since a fresh Haar-random frame is sampled at every step, the
corresponding expected slice objective is
\[
\bar{\mathcal R}_{d'}(\boldsymbol{\mu},\boldsymbol{\Sigma})
=
\mathbb E_{\mathbf{U}}[\mathcal R_{\mathbf{U}}(\boldsymbol{\mu},\boldsymbol{\Sigma})].
\]

\begin{proposition}[Random slicing]
For every $1\le d'\le d$:

\begin{enumerate}[leftmargin=*]
\item[(i)] $\bar{\mathcal R}_{d'}\ge0$, with equality if and only if
$\boldsymbol{\mu}=\mathbf{0}$ and $\boldsymbol{\Sigma}=\mathbf{I}_d$.

\item[(ii)] The normalized mean and total variance are preserved in
expectation:
\[
\mathbb E_{\mathbf{U}}\frac{\|\mathbf{U}^\top\boldsymbol{\mu}\|_2^2}{d'}
=
\frac{\|\boldsymbol{\mu}\|_2^2}{d},
\qquad
\mathbb E_{\mathbf{U}}\frac{\operatorname{tr}(\mathbf{U}^\top\boldsymbol{\Sigma}\mathbf{U})}{d'}
=
\frac{\operatorname{tr}\boldsymbol{\Sigma}}{d}.
\]

\item[(iii)] The expected sliced objective approaches the full one
monotonically:
\[
\bar{\mathcal R}_{d'}
\le
\bar{\mathcal R}_{d'+1}
\le
\mathcal R_{\mathrm{full}},
\qquad
\bar{\mathcal R}_{d}=\mathcal R_{\mathrm{full}}.
\]
\end{enumerate}
\end{proposition}

Thus, slicing does not change the preferred representation: averaging
over random subspaces still uniquely favors zero mean and isotropic
unit covariance. The mean and total variance are preserved exactly in
expectation, while the log-determinant observes only the spectrum inside
the sampled subspace. Wider slices therefore capture progressively more
of the full spectral penalty, recovering the original regularizer when
$d'=d$. In practice, we choose $d'<d$ so that the projected covariance
can be estimated from substantially more centers than dimensions.

\paragraph{Proof.}
Each $\mathcal R_{\mathbf{U}}$ is nonnegative and is minimized when
$\mathbf{U}^\top\boldsymbol{\mu}=\mathbf{0}$ and $\mathbf{U}^\top\boldsymbol{\Sigma}\mathbf{U}=\mathbf{I}_{d'}$. If this holds for all random
subspaces, then every one-dimensional direction has zero projected mean
and unit variance, implying $\boldsymbol{\mu}=\mathbf{0}$ and $\boldsymbol{\Sigma}=\mathbf{I}_d$.

Haar invariance gives
\[
\mathbb E_{\mathbf{U}}[\mathbf{U}\mathbf{U}^\top]=\frac{d'}{d}\mathbf{I}_d,
\]
which directly yields (ii). Finally, concavity of the logarithm gives
\[
\frac{1}{d'}\mathbb E_{\mathbf{U}}
\log\det(\mathbf{U}^\top\boldsymbol{\Sigma}\mathbf{U})
\ge
\frac{1}{d}\log\det\boldsymbol{\Sigma}.
\]
Together with (ii), this gives
$\bar{\mathcal R}_{d'}\le\mathcal R_{\mathrm{full}}$.
Applying the same argument to a random $d'$-dimensional subspace inside
a random $(d'+1)$-dimensional subspace gives
$\bar{\mathcal R}_{d'}\le\bar{\mathcal R}_{d'+1}$.
For $d'=d$, $\mathbf{U}$ is orthogonal and the two objectives are identical.

\paragraph{Ring buffer.}
Reliable covariance estimation requires more centers than the slice dimension.
For each regularized space, we therefore maintain a FIFO buffer containing detached view centers from the previous $q$ steps and concatenate them with the current batch before computing the regularizer.
This gives
\[
B_{\mathrm{eff}}=(q+1)B,
\]
where only the current $B$ centers carry gradients.
Under data parallelism, the buffer contains centers from the full global batch.

We keep $B_{\mathrm{eff}}/d'=4$ throughout image training.
Thus, $q=3$ after the projector and at the ViT-S backbone
($B_{\mathrm{eff}}=512$, $d'=128$), while $q=7$ at the ViT-B backbone
($B_{\mathrm{eff}}=1024$, $d'=256$).
On video, the larger batch sizes reduce the need for buffering:
no buffer is used for ViT-S, while ViT-B uses $q=1$ at the backbone.
We observed no measurable effect from the short parameter lag introduced by the buffered centers.

\paragraph{Averaged invariance target.}
For ImageNet-1k and video, we follow LeJEPA and compute the invariance target using an exponential moving average of the backbone and projector. The target $\tilde{\mathbf{z}}_i$ is the mean projected representation of the $V$ views under the averaged network, and we use
\[
\mathcal L_{\mathrm{SSL}}
=
\frac{2V}{V-1}\,
\frac{1}{B V d_z}
\sum_{i=1}^{B}\sum_{j=1}^{V}
\left\|\mathbf{z}_i^{(j)}-\tilde{\mathbf{z}}_i\right\|_2^2 .
\]

ImageNet-100 uses the pairwise objective directly. In all cases, SACReg is computed from view centers of the current network.

\section{Experimental Details}
\label{app:exp-details}

\subsection{Datasets}
ImageNet-1k~\citep{imagenetrussakovsky2015} is the standard split with 1.28M training and 50k validation images. ImageNet-100 is the 100-class subset introduced by CMC~\citep{inet100tian2020contrastive}, with 126{,}689 training and 5{,}000 validation images. For video we follow LeVJEPA~\citep{levjepakuhn2026}: we take the union of the Kinetics-400, -600 and -700~\citep{kinetics710li2023uniformerv2} training sets, remove duplicates and every clip that appears in a validation set, and keep 20\% of the clips of each class (120{,}100 clips after decoding). LeVJEPA does not release its subset, so ours is a different draw of the same size. Transfer~\citep{wellericsson2021} uses DTD~\citep{dtdcimpoi2014describing}, Aircraft~\citep{aircraftmaji2013fine}, Cars~\citep{carskrause20133d}, CIFAR10, CIFAR100~\citep{cifarkrizhevsky2009learning}, Flowers~\citep{flowersnilsback2008automated}, Food~\citep{foodbossard2014} and Pets~\citep{petsparkhi2012cats} with the splits of the VISReg protocol~\citep{visreg2026}; video encoders are evaluated on ImageNet-1k, Something-Something-v2~\citep{ssv2goyal2017something} and Kinetics-400~\citep{kinetics400kay2017}.

\subsection{Models and training}
\paragraph{Large-scale image runs.}
We train ViT-S/16 and ViT-B/16 at $224^2$ with a two-layer projector (hidden width 2048, output 256, BatchNorm), batch size 128 split over two GPUs, for 100 and 400 epochs. Each image gives $V=6$ global views from the LeJEPA augmentation family: random resized crops covering 30--100\% of the image, horizontal flips, color jitter, random grayscale, Gaussian blur, and solarization on one of the views. We do not use local crops. The slices and buffers are as described in Appendix~\ref{app:slicing}. Optimization uses AdamW with learning rate $10^{-3}$, weight decay $0.05$, 10 warm-up epochs followed by a cosine schedule, gradient clipping at 1.0, and bf16 precision. The loss weights are $(31.07,\,225.50,\,1.671)$ for ViT-S and $(22.81,\,222.26,\,4.054)$ for ViT-B, for the invariance term, $\beta_z$ and $\beta_h$ respectively; they are set as described below and shared by the 100- and 400-epoch runs.

\paragraph{Controlled experiments on ImageNet-100.}
All controlled experiments use ViT-S/16 at $224^2$, are trained for 100 epochs with batch size 128, and are repeated over three independent random seeds. We re-implement LeJEPA, VICReg, SimCLR, DINO, BYOL, and VISReg within a common training codebase while preserving each method's released loss, projector, and augmentation settings. For each method, we train a baseline and a matched variant with $\beta_h\operatorname{SACReg}(\bar{\mathbf{H}})$ added; all other settings are kept fixed, so the two variants differ only by the backbone regularization term. We report the mean across the three seeds. The regularizer is applied to the representation provided as input to each method's projector: the CLS token for SimCLR, BYOL, VICReg, and DINO, and the 512-dimensional embedding immediately preceding the projector for LeJEPA and VISReg. Our own objective uses $V=4$ views and the pairwise invariance term defined in Appendix~\ref{app:slicing}, with weights $32.8$ for invariance, $157.8$ for $\beta_z$, and $1.894$ for $\beta_h$. The corresponding baseline sets $\beta_h=0$, with all other settings unchanged.

\paragraph{Video.}
We use LeVJEPA's training pipeline and released configuration unchanged except for the loss: their video ViT-S/16 and ViT-B/16 (16 frames at stride 2, tubelet size 1, 95\% token drop), their projector, and their optimizer (AdamW, learning rate $4\times10^{-4}$, weight decay $0.04$, warm-up over 12\% of training followed by a constant learning rate, bf16). Our loss takes $V=6$ global $224^2$ views of the same clip. The batch size is 768 for ViT-S and 512 for ViT-B, the loss weights are those of the image run of the same model size, and we evaluate the exponential-moving-average encoder, as LeVJEPA does.

\subsection{Setting the loss weights}
\label{app:loss-weights}
The loss terms have different forms and scales, so we do not balance them by their values but by their realized pull on the backbone. For a term with weight $w$, let $g$ be the norm of the gradient of the unweighted term with respect to the backbone parameters, measured on a fixed batch at a given checkpoint. The product $w\,g$ is the pull of the term, and its share of the total pull is $w\,g/\sum_k w_k g_k$.

\paragraph{Adding the backbone term to an existing method.}
When adding $\beta_h\operatorname{SACReg}(\bar{\mathbf{H}})$ to an existing SSL method, we choose $\beta_h$ so that the backbone regularizer contributes a small fraction $T=0.06$ of the sum of weighted backbone-gradient norms. Let
\[
G_{\mathrm{SSL}}=\sum_k w_k g_k
\]
denote the contribution of the method's original loss terms, measured at its epoch-25 checkpoint. We set
\[
\beta_h=\frac{T}{1-T}\frac{G_{\mathrm{SSL}}}{g_R},
\]
using a fixed reference $g_R=0.30$ for the gradient norm of $\operatorname{SACReg}$. We use a fixed reference because this gradient can be unusually large at a collapsed backbone, which would otherwise yield an artificially small weight. The target share $T=0.06$ was chosen from an earlier dose study on LeJEPA. Thus, $T$ and $g_R$ are fixed across methods, and $\beta_h$ is determined from the scale of each method's existing objective rather than tuned separately.

This gives $\beta_h=0.098$ (SimCLR), $0.0205$ (BYOL), $1.548$ (VICReg), $0.258$ (DINO), $0.0146$ (LeJEPA), and $0.0123$ (VISReg; measured from a one-epoch pilot). For DINO, this weight reduced both linear and kNN accuracy, so we tested $\{0.115,0.02,0.01\}$ and use $\beta_h=0.01$, selected by the online probe.

\paragraph{Our standalone objective.}
For our standalone objective, we set the three loss weights jointly. After a one-epoch pilot, we measure the backbone gradient norm $g_k$ of each term and choose
\[
w_k = 57\,\frac{s_k}{g_k},
\]
with target shares $(s_{\mathrm{inv}},s_{\mathrm{proj}},s_{\mathrm{back}})=(0.44,0.54,0.02)$. These shares and the total magnitude $57$ come from an earlier ViT-S run, and we use the same calibration rule for ViT-S and ViT-B without further tuning.

\subsection{Evaluation}
\paragraph{ImageNet-1k linear probe and kNN.}
We follow the linear-evaluation recipe of the Lightly benchmark (the MAE linear-probe recipe): a BatchNorm layer without affine parameters followed by a linear classifier on the frozen CLS token, trained for 90 epochs with random resized crops and flips, SGD with momentum 0.9 (equivalent to LARS at zero weight decay), learning rate $0.1\times\mathrm{batch}/256$ with a cosine schedule and 10 warm-up epochs; we report the best validation top-1 over the 90 epochs. The kNN classifier uses $k=200$ neighbors with cosine similarity weighted at temperature $0.07$, on the CLS features of the training set.

\paragraph{Transfer.}
We use the VISReg transfer protocol: the CLS tokens of the last four blocks are concatenated, a linear classifier is trained for 10 epochs over a grid of 13 learning rates, and the best test accuracy is reported for each dataset; the table reports the average over the eight datasets. Our implementation reproduces the published DINO ViT-B/16 row within 0.4 points on every dataset.

\paragraph{ImageNet-100.}
The linear probe is a linear classifier on the raw CLS features, trained with AdamW (learning rate $10^{-3}$, weight decay $10^{-7}$) until validation accuracy stops improving; the kNN classifier uses $k=200$ and temperature $0.1$. Both are fit on a fixed subset of 500 training images per class and evaluated on the 5{,}000 validation images.

\paragraph{Video.}
We follow the frozen-encoder protocols of V-JEPA as used by LeVJEPA. ImageNet-1k and Something-Something-v2 use an attentive probe: one cross-attention block with a learnable query over all output tokens followed by a linear classifier, trained for 20 epochs with AdamW (learning rate $10^{-3}$, cosine schedule). For ImageNet-1k each image is repeated over the 16 input frames; for Something-Something-v2 a video is sampled as 2 clips of 16 frames with 3 spatial crops. Kinetics-400 uses a linear classifier. LeVJEPA does not release the settings of this linear classifier. So we report the results with L-BFGS (multinomial logistic regression on standardized features, the best of three ridge strengths $\{10^{-5},10^{-4},10^{-3}\}$) using one center crop per clip and the full training and validation sets.

\paragraph{Representation statistics.}
All statistics are computed on frozen features. RankMe~\citep{rankme2023} is the exponential of the entropy of the normalized singular values reported as a fraction of the feature dimension. Positive and negative-pair cosines are the mean cosine similarity between two views of the same image and between views of two different images.

\paragraph{Augmentation thickness.}
For every model we store $V=8$ augmented views of 10,000 training images
(100 per class), drawn from the model's own training augmentations. For the
per-class analysis of Fig.~\ref{fig:organization-sensitivity}, we split the
10,000 images once at random into two equal halves. On the fitting half, we
estimate $\mathbf{A}_h$ as the average within-image covariance of the views and $\mathbf{B}_h$
as the covariance of the view means minus $\mathbf{A}_h/V$, which corrects for
estimating each image center from finitely many views. We use $\mathbf{B}_h$ to define
the whitening map and compute
$\boldsymbol{\Theta}_h = \mathbf{B}_h^{\dagger/2} \mathbf{A}_h \mathbf{B}_h^{\dagger/2}$.

For each class $c$, we define the target as the centered indicator scaled to
unit variance,
\[
y_c
=
\frac{\mathbf{1}[y=c]-p_c}{\sqrt{p_c(1-p_c)}},
\]
where $p_c$ is the class frequency in the fitting split. Using the fitting
half, we fit the least-squares coefficient $\mathbf{a}_c$ for predicting $y_c$ from
the whitened image centers. We then define $d_h(c)^2$ as the mean squared
residual on the held-out half. Since this is a held-out error, it can exceed
$1$, so the class-center separability $1-d_h(c)^2$ can be negative for some
classes. The view-sensitivity term
$\mathbf{a}_c^\top \boldsymbol{\Theta}_h
(\mathbf{I}_{r_h}+\boldsymbol{\Theta}_h)^{-1}\mathbf{a}_c$
uses the same $\mathbf{a}_c$ and the $\boldsymbol{\Theta}_h$ estimated on the fitting half.
% \paragraph{Augmentation thickness.}
% For every model we store $V=8$ augmented views of 10,000 training images
% (100 per class), drawn from the model's own training augmentations. For the
% per-class analysis of Fig.~\ref{fig:organization-sensitivity}, we split the
% 10,000 images once at random into two equal halves. On the fitting half, we
% estimate $A_h$ as the average within-image covariance of the views and $B_h$
% as the covariance of the view means minus $A_h/V$, which corrects for
% estimating each image center from finitely many views. We use $B_h$ to define
% the whitening map and compute $\Theta_h = B_h^{\dagger/2} A_h B_h^{\dagger/2}$.

% For each class $c$, we define the target as the centered indicator scaled to
% unit variance,
% \[
% y_c
% =
% \frac{\mathbf{1}[y=c]-p_c}{\sqrt{p_c(1-p_c)}},
% \]
% where $p_c$ is the class frequency in the fitting split. Using the fitting
% half, we fit the least-squares coefficient $a_c$ for predicting $y_c$ from
% the whitened image centers. We then define $d_h(c)^2$ as the mean squared
% residual on the held-out half. Since this is a held-out error, it can exceed
% $1$, so the class-center separability $1-d_h(c)^2$ can be negative for some
% classes. The view-sensitivity term $a_c^\top \Theta_h (I_{r_h}+\Theta_h)^{-1}a_c$ uses the same $a_c$ and the $\Theta_h$ estimated on the fitting half.

\subsection{Baselines}
OK-AI releases LeJEPA, DINO and iBOT trained under one modernized code base at ViT-S and ViT-B for 100 and 300 epochs; we evaluate these checkpoints under the protocols above. Their model card lists 1.45M training images, which is more than the standard ImageNet-1k training split; we did not correct for this. We also trained VISReg ViT-B/16 for 100 epochs with the authors' code and their released view configuration (four global and six local crops), and evaluated the public MoCo~v3, DINO, iBOT and VISReg checkpoints (300--400 epochs) the same way. Transfer numbers marked $^\ddag$ are taken from the VISReg paper. Video baselines are taken from the LeVJEPA paper: the 240-epoch numbers from its Fig.~2 and the remaining rows from its equal-compute table.

\subsection{Detailed results}
\label{app:detailed-results}
\begin{table}[!h]
\centering
\caption{Per-dataset linear-probe transfer accuracy (\%) under the VISReg protocol; the last column is the average reported in Table~\ref{tab:main-results} and~\ref{tab:long-training-results}. $^\ddag$Numbers reported by the VISReg paper; all other rows are evaluated by us under the matched protocol.}
\label{tab:transfer-full}
\resizebox{\textwidth}{!}{%
\begin{tabular}{llrrrrrrrrrr}
\toprule
Method & Backbone & Ep. & DTD & Aircraft & Cars & CIFAR10 & CIFAR100 & Flowers & Food & Pets & Avg. \\
\midrule
LeJEPA & ViT-S/16 & 100 & 69.4 & 43.8 & 43.4 & 89.3 & 71.1 & 84.0 & 73.6 & 75.3 & 68.7 \\
DINO & ViT-S/16 & 100 & 69.3 & 55.2 & 57.0 & 93.8 & 78.6 & 89.3 & 76.7 & 85.9 & 75.7 \\
iBOT & ViT-S/16 & 100 & 69.3 & 55.3 & 56.1 & 93.3 & 77.4 & 90.3 & 77.9 & 87.3 & 75.9 \\
\rowcolor{oursblue}
$\lambda$-JEPA & ViT-S/16 & 100 & 71.1 & 55.5 & 62.5 & 95.3 & 81.6 & 90.4 & 75.3 & 89.1 & 77.6 \\
\midrule
LeJEPA & ViT-S/16 & 300 & 70.2 & 47.6 & 50.3 & 92.1 & 74.4 & 85.8 & 75.7 & 79.1 & 71.9 \\
DINO & ViT-S/16 & 300 & 71.5 & 59.8 & 66.2 & 95.0 & 81.1 & 92.0 & 79.6 & 89.6 & 79.4 \\
iBOT & ViT-S/16 & 300 & 71.6 & 59.7 & 63.3 & 96.0 & 82.4 & 91.2 & 80.9 & 90.6 & 79.5 \\
\rowcolor{oursblue}
$\lambda$-JEPA & ViT-S/16 & 400 & 71.8 & 57.9 & 67.9 & 95.5 & 81.8 & 90.3 & 78.1 & 90.7 & 79.3 \\
\midrule
LeJEPA & ViT-B/16 & 100 & 72.3 & 51.5 & 54.4 & 93.0 & 76.4 & 86.9 & 78.9 & 80.1 & 74.2 \\
DINO & ViT-B/16 & 100 & 70.0 & 58.1 & 61.9 & 95.8 & 82.0 & 89.9 & 79.7 & 88.7 & 78.3 \\
iBOT & ViT-B/16 & 100 & 71.9 & 61.1 & 66.8 & 97.0 & 84.2 & 91.5 & 82.8 & 90.1 & 80.7 \\
VISReg & ViT-B/16 & 100 & 73.3 & 53.1 & 58.0 & 94.3 & 79.6 & 87.9 & 79.1 & 85.5 & 76.4 \\
\rowcolor{oursblue}
$\lambda$-JEPA & ViT-B/16 & 100 & 71.5 & 59.8 & 71.1 & 96.7 & 84.9 & 92.2 & 79.4 & 91.7 & 80.9 \\
\midrule
LeJEPA & ViT-B/16 & 300 & 73.7 & 52.0 & 55.0 & 93.5 & 77.6 & 87.7 & 80.5 & 81.8 & 75.2 \\
DINO & ViT-B/16 & 300 & 71.8 & 59.1 & 63.8 & 96.0 & 83.4 & 90.0 & 80.5 & 89.9 & 79.3 \\
iBOT & ViT-B/16 & 300 & 74.5 & 62.0 & 68.8 & 97.4 & 86.2 & 93.0 & 84.2 & 92.2 & 82.3 \\
MoCo v3$^\ddag$ & ViT-B/16 & 300 & 73.7 & 57.9 & 67.5 & 96.9 & 85.2 & 91.5 & 81.8 & 89.8 & 80.5 \\
DINO$^\ddag$ & ViT-B/16 & 400 & 74.3 & 63.6 & 73.9 & 96.5 & 85.0 & 94.6 & 83.1 & 93.6 & 83.1 \\
iBOT$^\ddag$ & ViT-B/16 & 400 & 74.1 & 63.5 & 73.8 & 97.1 & 85.9 & 93.7 & 84.2 & 93.6 & 83.2 \\
VISReg$^\ddag$ & ViT-B/16 & 400 & 75.7 & 57.1 & 64.8 & 94.6 & 78.8 & 90.4 & 82.9 & 88.3 & 79.1 \\
\rowcolor{oursblue}
$\lambda$-JEPA & ViT-B/16 & 400 & 73.2 & 61.1 & 74.0 & 96.9 & 85.4 & 92.0 & 81.0 & 92.2 & 82.0 \\
\bottomrule
\end{tabular}%
}
\end{table}

\begin{figure*}[!h]
    \centering
    \includegraphics[width=0.8\textwidth]{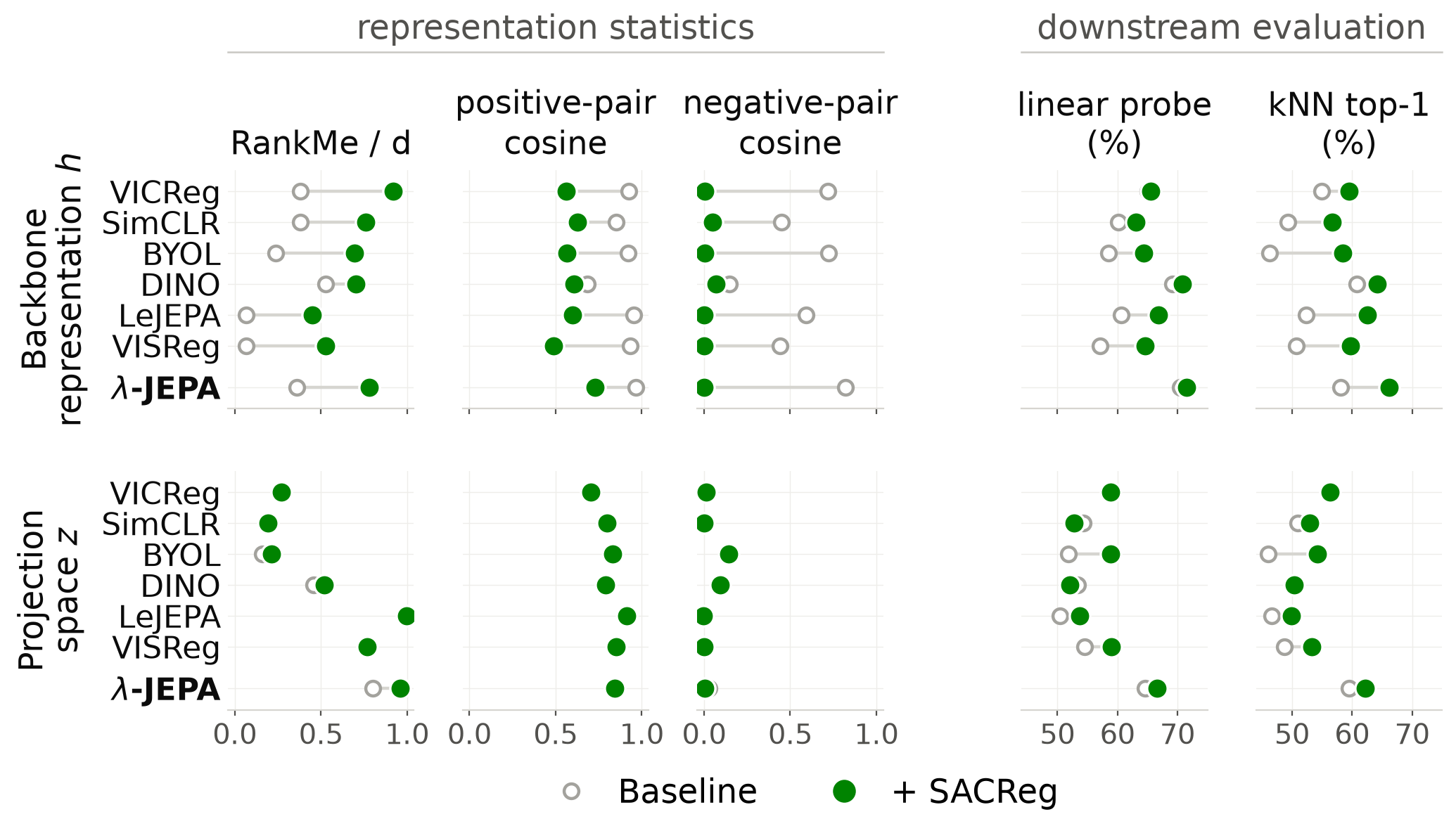}
    \caption{
    Effect of applying SACReg to the backbones across SSL objectives on ImageNet-100. We report representation statistics at the backbone $\mathbf{h}$ and projection space $\mathbf{z}$, together with downstream linear probe and kNN accuracy. Open circles denote the baseline method and green circles the same method with SACReg added at the backbone. Each point is the mean over three seeds.
    }
    \label{fig:two-space-audit}
\end{figure*}

% =========================
% Organization vs sensitivity
% =========================
\begin{figure*}[!h]
    \centering
    \includegraphics[width=0.75\textwidth]{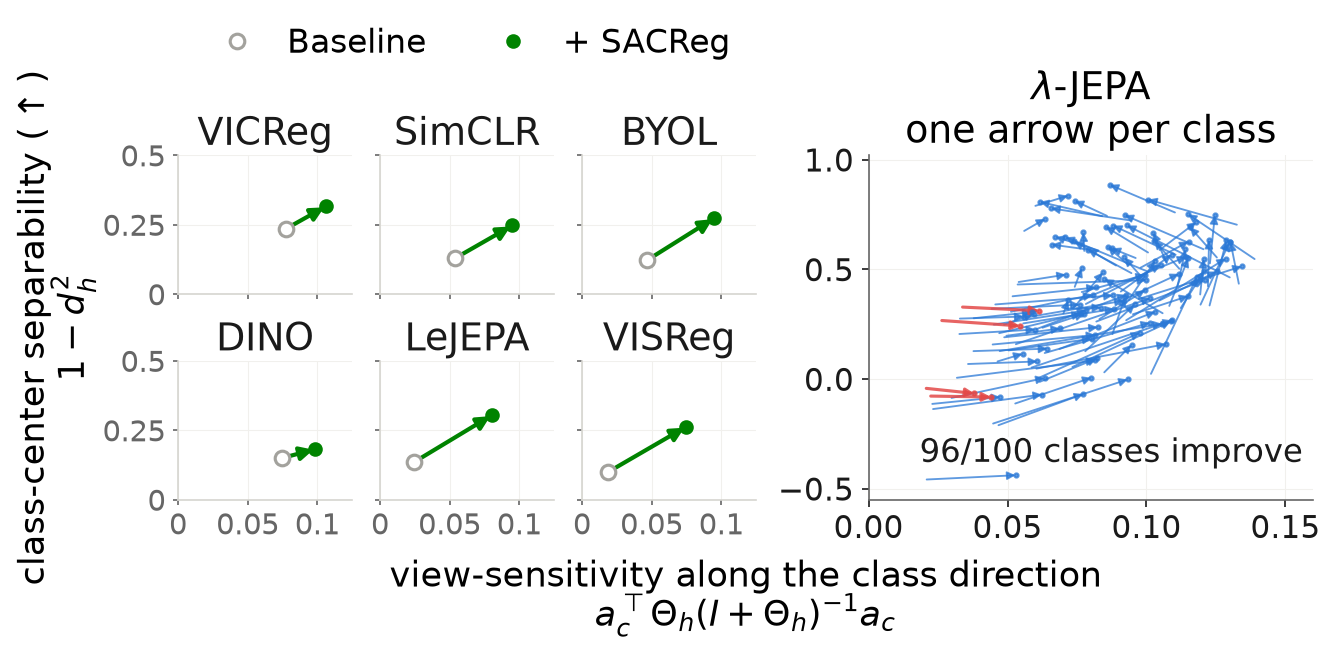}
    % \caption{
    % Image-center organization and view sensitivity along class-prediction directions. Left: average over classes for each method before and after backbone regularization. Right: class-wise changes for our model.
    % }
    \caption{
    Effect of SACReg on view sensitivity and class-center separability. Left: mean class-level changes for six SSL methods on ImageNet-100, with arrows from the baseline representation to the same method with SACReg added at the backbone. Right: class-wise changes for $\lambda$-JEPA; 96 of 100 classes improve in class-center separability.
    }
    \label{fig:organization-sensitivity}
\end{figure*}

\section{A closer look at LeJEPA}
\label{app:lejepa}

Our results suggest that anti-collapse should act on the representation retained for downstream tasks. A natural alternative explanation is that the particular regularizer is unimportant: perhaps simply applying an existing anti-collapse objective at the backbone is sufficient. We test this with LeJEPA, one of the explicitly regularized methods in our study.

LeJEPA encourages the projected representation to match an isotropic Gaussian using SIGReg, a sketched Gaussianity test over random one-dimensional projections. However, its retained backbone representation is far from this target (Fig.~\ref{fig:hz_geometry}). Starting from vanilla LeJEPA, with the original SIGReg and invariance losses after the projector unchanged, we therefore add a second SIGReg term directly to the 512-dimensional encoder output, using the weighting heuristic of Appendix~\ref{app:loss-weights}, and train under the ImageNet-100 setup of Section~\ref{sec:exp-in100}.

Table~\ref{tab:lejepa-sigreg} shows that the added SIGReg substantially improves several Gaussianity statistics at the backbone. Negative-pair cosine drops from $0.58$ to $0.01$, the diagonal Gaussian KL from $0.96$ to $0.07$, and the Epps--Pulley statistic also decreases. However, the covariance spectrum remains concentrated: RankMe$/d$ changes only from $0.07$ to $0.09$, while the full-covariance KL remains high ($3.89\to3.03$). At the same time, positive-pair cosine decreases and downstream performance worsens by $1.4$ linear-probe points and $4.3$ kNN points. Thus, the gains from SACReg are not explained simply by moving LeJEPA's existing anti-collapse constraint to the retained representation. In this comparison, directly controlling the covariance spectrum is important: SIGReg substantially improves its Gaussianity diagnostics but leaves the backbone spectrum concentrated, whereas $\mathcal{R}_{\mathrm{SAC}}$ acts explicitly on that spectrum through its log-determinant term.
\begin{table}[h]
\centering
\small
\begin{tabular}{lrr}
\toprule
Readout at encoder output ($\mathbf{h}$) & LeJEPA & + SIGReg at $\mathbf{h}$ \\
\midrule
Positive-pair cosine & 0.96 & 0.73 \\
Negative-pair cosine & 0.58 & 0.01 \\
KL to $\mathcal N(0,1)$ (diagonal) & 0.96 & 0.07 \\
KL to $\mathcal N(0,I)$ (full covariance) & 3.89 & 3.03 \\
Epps--Pulley & 652 & 504 \\
RankMe$/d$ & 0.07 & 0.09 \\
Linear probe (\%) & 60.4 & 58.9 \\
kNN, $k=200$ (\%) & 52.4 & 48.1 \\
\bottomrule
\end{tabular}
\caption{Adding SIGReg directly to the LeJEPA encoder output on ImageNet-100. The original projector losses are unchanged. All representation statistics are measured at the 512-dimensional encoder output.}
\label{tab:lejepa-sigreg}
\end{table}

\section{Augmentation thickness: definition and interpretation}
\label{app:thickness}

This section gives the formal motivation for the augmentation-thickness quantities used in Section~\ref{sec:exp-in100}. In particular, we show that augmentation thickness has two opposing effects: too little thickness can suppress meaningful augmentation-dependent variation,
whereas too much allows within-image variation to dominate the separation
between images.

\subsection{Definition}

Let $Q$ denote a source image and let $\mathbf{h}$ be its representation under a randomly sampled augmented view. Define the image center
\begin{equation}
    \mathbf{m}_h(Q) = \mathbb{E}[\mathbf{h}\mid Q]
\end{equation}
and the residual view-dependent component
\begin{equation}
    \boldsymbol{\xi}_{h} = \mathbf{h}-\mathbf{m}_h(Q).
\end{equation}
By the law of total covariance,
\begin{equation}
    \mathbf{C}_h = \operatorname{Cov}(\mathbf{h}) = \mathbf{B}_h + \mathbf{A}_h,
    \qquad
    \mathbf{B}_h = \operatorname{Cov}\!\left(\mathbf{m}_h(Q)\right),
    \qquad
    \mathbf{A}_h = \mathbb{E}\!\left[\operatorname{Cov}(\mathbf{h}\mid Q)\right].
    \label{eq:thickness-decomposition}
\end{equation}
Thus, $\mathbf{B}_h$ measures variation between image centers with the rank $r_h = \operatorname{rank}(\mathbf{B}_h)$, while $\mathbf{A}_h$ measures variation across augmented views of the same image.

The same amount of within-image variation can have very different consequences depending on the scale of the between-image variation. We therefore normalize $\mathbf{A}_h$ by $\mathbf{B}_h$ and define
\begin{equation}
    \boldsymbol{\Theta}_h
    =
    \mathbf{B}_h^{\dagger/2} \mathbf{A}_h \mathbf{B}_h^{\dagger/2}.
    \label{eq:thickness-app}
\end{equation}
We call $\boldsymbol{\Theta}_h$ the \emph{augmentation thickness} of the representation. Along any direction $\mathbf{a}$ with nonzero between-image variance,
\begin{equation}
    \theta_h(\mathbf{a})
    =
    \frac{\mathbf{a}^\top \mathbf{A}_h \mathbf{a}}{\mathbf{a}^\top \mathbf{B}_h \mathbf{a}}
\end{equation}
is the ratio between within-image and between-image variation in that direction.

We next formalize why augmentation thickness should not simply be minimized. Let $y=y(Q)$ be a centered image-level target with $\|y\|_{L_2}\leq 1$, and define
\begin{equation}
    E_{\mathcal A}(y)
    =
    \min_f
    \mathbb{E}\!\left[
        \left(y(Q)-f(\mathbf{v})\right)^2
    \right],
    \label{eq:augmentation-irreducible}
\end{equation}
where $\mathbf{v}$ is a single augmented view of $Q$. This is the irreducible error of predicting $y$ from one augmented view.

Let $\mathcal M_h$ denote the set of affine functions of the image centers and define
\begin{equation}
    d_h(y)
    =
    \min_{g\in\mathcal M_h}
    \|y-g\|_{L_2}.
    \label{eq:center-prediction-error}
\end{equation}
Thus, $d_h(y)^2$ is the error of the best affine prediction of $y$ from the image centers.

\begin{theorem}[Two sides of augmentation thickness]
\label{thm:thickness}
Restrict to the support of $\mathbf{B}_h$ and work in centered, whitened coordinates, so that
\begin{equation}
    \mathbb{E}[\mathbf{m}_h(Q)]=\mathbf{0},
    \qquad
    \operatorname{Cov}(\mathbf{m}_h(Q))=\mathbf{I}_{r_h},
    \qquad
    \operatorname{Cov}(\boldsymbol{\xi}_{h})=\boldsymbol{\Theta}_h.
\end{equation}
Then
\begin{equation}
    d_h(y)
    \geq
    \left(
        \sqrt{E_{\mathcal A}(y)}
        -
        \sqrt{\|\boldsymbol{\Theta}_h\|_{\mathrm{op}}}
    \right)_+ .
    \label{eq:too-thin}
\end{equation}

Moreover, let
\begin{equation}
    g(Q)=\mathbf{a}^\top \mathbf{m}_h(Q)
\end{equation}
be the best linear prediction of $y$ from the image centers. Then the best linear prediction error from a single-view representation $\mathbf{h}$ is
\begin{equation}
    \inf_{\mathbf{u}}
    \mathbb{E}\!\left[(y-\mathbf{u}^\top \mathbf{h})^2\right]
    =
    d_h(y)^2
    +
    \mathbf{a}^\top
    \boldsymbol{\Theta}_h (\mathbf{I}_{r_h}+\boldsymbol{\Theta}_h)^{-1} \mathbf{a}.
    \label{eq:too-thick}
\end{equation}
For fixed image-center organization and fixed task direction $\mathbf{a}$, the second term is monotone increasing in $\boldsymbol{\Theta}_h$ in the positive-semidefinite order.
\end{theorem}

The two statements capture complementary failure modes. Equation~\ref{eq:too-thin} says that if a target cannot be recovered perfectly from an augmented view, then a representation with very little augmentation thickness cannot nevertheless organize its image centers arbitrarily well for that target. Some augmentation-dependent variation may therefore need to remain in the representation. In the opposite direction, Eq.~\ref{eq:too-thick} shows that, once the image-center organization is fixed, additional within-image variation makes prediction from a single view harder. Augmentation thickness therefore describes a trade-off rather than a quantity that should be uniformly minimized.

The quantity
\begin{equation}
    \mathbf{a}^\top \boldsymbol{\Theta}_h(\mathbf{I}_{r_h}+\boldsymbol{\Theta}_h)^{-1}\mathbf{a}
    \label{eq:view-sensitivity}
\end{equation}
is the \emph{view-sensitivity} term reported in Fig.~\ref{fig:organization-sensitivity}. We use this quantity in Fig.~\ref{fig:organization-sensitivity} rather than the raw directional thickness $\theta_h(\mathbf{a})$ because it captures the contribution of augmentation variation after optimizing the linear prediction. It isolates the part of the optimal single-view linear prediction error attributable to variation across augmentations, separately from the image-center organization term $d_h(y)^2$.

\subsection{Proof of Theorem~\ref{thm:thickness}}

Throughout the proof we work on the support of $\mathbf{B}_h$ in centered, whitened coordinates and reuse $\mathbf{h}$, $\mathbf{m}_h$, and $\boldsymbol{\xi}_{h}$ for the transformed variables. Hence
\begin{equation}
    \mathbb{E}[\mathbf{m}_h(Q)]=\mathbf{0},
    \qquad
    \operatorname{Cov}(\mathbf{m}_h(Q))=\mathbf{I}_{r_h},
    \qquad
    \mathbb{E}[\boldsymbol{\xi}_{h}\mid Q]=\mathbf{0},
    \qquad
    \operatorname{Cov}(\boldsymbol{\xi}_{h})=\boldsymbol{\Theta}_h.
\end{equation}

\paragraph{Too little thickness.}
Consider any linear function of the image centers,
\begin{equation}
    g(Q)=\mathbf{a}^\top \mathbf{m}_h(Q).
\end{equation}
Since $\mathbf{a}^\top \mathbf{h}$ is computable from a single augmented view, it is one admissible predictor of $g$. Therefore,
\begin{align}
    E_{\mathcal A}(g)
    &\leq
    \mathbb{E}\!\left[
        \left(\mathbf{a}^\top \mathbf{m}_h(Q)-\mathbf{a}^\top \mathbf{h}\right)^2
    \right] \\
    &=
    \mathbb{E}\!\left[(\mathbf{a}^\top\boldsymbol{\xi}_{h})^2\right] \\
    &=
    \mathbf{a}^\top\boldsymbol{\Theta}_h \mathbf{a}.
    \label{eq:app-thin-1}
\end{align}

Now let $g$ be the best linear prediction of $y$ from the image centers. Since the image-center covariance is whitened, we may write
\begin{equation}
    g(Q)=\mathbf{a}_y^\top \mathbf{m}_h(Q).
\end{equation}
Orthogonal projection in $L_2$ gives
\begin{equation}
    \|\mathbf{a}_y\|
    =
    \|g\|_{L_2}
    \leq
    \|y\|_{L_2}
    \leq 1.
\end{equation}
Using the triangle inequality,
\begin{align}
    \sqrt{E_{\mathcal A}(y)}
    &\leq
    \|y-g\|_{L_2}
    +
    \sqrt{E_{\mathcal A}(g)} \\
    &\leq
    d_h(y)
    +
    \sqrt{\mathbf{a}_y^\top\boldsymbol{\Theta}_h \mathbf{a}_y} \\
    &\leq
    d_h(y)
    +
    \sqrt{\|\boldsymbol{\Theta}_h\|_{\mathrm{op}}},
\end{align}
where the last inequality uses $\|\mathbf{a}_y\|\leq1$. Rearranging yields
\begin{equation}
    d_h(y)
    \geq
    \left(
        \sqrt{E_{\mathcal A}(y)}
        -
        \sqrt{\|\boldsymbol{\Theta}_h\|_{\mathrm{op}}}
    \right)_+,
\end{equation}
which proves Eq.~\ref{eq:too-thin}.

\paragraph{Too much thickness.}
Let
\begin{equation}
    g(Q)=\mathbf{a}^\top \mathbf{m}_h(Q)
\end{equation}
be the best linear prediction of $y$ from the image centers, and write
\begin{equation}
    y=g+r.
\end{equation}
By the orthogonality property of linear regression,
\begin{equation}
    \mathbb{E}[r\,\mathbf{m}_h(Q)]=\mathbf{0},
    \qquad
    \mathbb{E}[r^2]=d_h(y)^2.
\end{equation}
Moreover, since $\mathbb{E}[\boldsymbol{\xi}_{h}\mid Q]=\mathbf{0}$ and $r$ is a function of $Q$,
\begin{equation}
    \mathbb{E}[r\,\boldsymbol{\xi}_{h}]
    =
    \mathbb{E}\!\left[
        r\,\mathbb{E}[\boldsymbol{\xi}_{h}\mid Q]
    \right]
    =\mathbf{0}.
\end{equation}
Hence $r$ is uncorrelated with
\begin{equation}
    \mathbf{h}=\mathbf{m}_h(Q)+\boldsymbol{\xi}_{h}
\end{equation}
and contributes exactly $d_h(y)^2$ to the optimal affine prediction error from $\mathbf{h}$.

For the component $g=\mathbf{a}^\top \mathbf{m}_h(Q)$,
\begin{equation}
    \operatorname{Cov}(\mathbf{h})
    =
    \mathbf{I}_{r_h}+\boldsymbol{\Theta}_h,
    \qquad
    \operatorname{Cov}(\mathbf{h},\mathbf{m}_h(Q))
    =
    \mathbf{I}_{r_h}.
\end{equation}
The best linear prediction error of $g$ from $\mathbf{h}$ is therefore
\begin{align}
    \mathbb{E}[g^2]
    &-
    \operatorname{Cov}(g,\mathbf{h})
    \operatorname{Cov}(\mathbf{h})^{-1}
    \operatorname{Cov}(\mathbf{h},g) \\
    &=
    \mathbf{a}^\top
    \left[\mathbf{I}_{r_h}-(\mathbf{I}_{r_h}+\boldsymbol{\Theta}_h)^{-1}\right]
    \mathbf{a} \\
    &=
    \mathbf{a}^\top
    \boldsymbol{\Theta}_h(\mathbf{I}_{r_h}+\boldsymbol{\Theta}_h)^{-1}
    \mathbf{a}.
\end{align}
Adding the residual contribution gives
\begin{equation}
    \inf_{\mathbf{u}}
    \mathbb{E}\!\left[(y-\mathbf{u}^\top \mathbf{h})^2\right]
    =
    d_h(y)^2
    +
    \mathbf{a}^\top
    \boldsymbol{\Theta}_h(\mathbf{I}_{r_h}+\boldsymbol{\Theta}_h)^{-1}\mathbf{a},
\end{equation}
which proves Eq.~\ref{eq:too-thick}.

Finally,
\begin{equation}
    \boldsymbol{\Theta}_h(\mathbf{I}_{r_h}+\boldsymbol{\Theta}_h)^{-1}
    =
    \mathbf{I}_{r_h}-(\mathbf{I}_{r_h}+\boldsymbol{\Theta}_h)^{-1}.
\end{equation}
If $\boldsymbol{\Theta}_h' \succeq \boldsymbol{\Theta}_h$, then
\begin{equation}
    (\mathbf{I}_{r_h}+\boldsymbol{\Theta}_h')^{-1}
    \preceq
    (\mathbf{I}_{r_h}+\boldsymbol{\Theta}_h)^{-1},
\end{equation}
and therefore
\begin{equation}
    \boldsymbol{\Theta}_h'(\mathbf{I}_{r_h}+\boldsymbol{\Theta}_h')^{-1}
    \succeq
    \boldsymbol{\Theta}_h(\mathbf{I}_{r_h}+\boldsymbol{\Theta}_h)^{-1}.
\end{equation}
Thus, for fixed $\mathbf{a}$, increasing augmentation thickness can only increase the view-sensitivity term. This completes the proof.

\end{document}